\documentclass[11pt]{article}

\usepackage[round]{natbib}
\usepackage[utf8]{inputenc} 
\usepackage[T1]{fontenc}    
\usepackage{url}            
\usepackage{booktabs}       
\usepackage{amsfonts}       
\usepackage{nicefrac}       
\usepackage{microtype}      
\usepackage{xcolor}         
\usepackage{wrapfig}
\usepackage[top=1in, left=1in, right=1in, bottom=1in]{geometry}
\usepackage[colorlinks,linkcolor=magenta,filecolor=blue,citecolor=blue,urlcolor=blue]{hyperref}%
\usepackage{graphicx}
\usepackage{bbm}
\usepackage[ruled]{algorithm2e}
\usepackage{algpseudocode}
\usepackage{amsthm}
\usepackage{authblk}
\usepackage{thmtools,thm-restate}
\usepackage{booktabs}
\usepackage{multirow}
\usepackage{subcaption}
\usepackage{amsmath}
\usepackage{amssymb}
\usepackage{pifont}
\usepackage{makecell}
\usepackage{multicol}
\usepackage{enumitem}
\newcommand{\sas}[1]{\texttt{WTA-CRS} #1}

\usepackage{amsmath,amsfonts}
\newcommand{\bm}[1]{\mathbf #1}

\newcommand{\relu}[1]{\textrm{ReLU} #1}
\newcommand{\gelu}[1]{\textrm{GeLU} #1}
\newcommand{\mm}[1]{\texttt{GEMM} #1}

\def\eqref#1{equation~(\ref{#1})}
\def\Eqref#1{Equation~(\ref{#1})}

\def\1{\bm{1}}

\def\vx{{\bm{x}}}
\def\vy{{\bm{y}}}
\def\vz{{\bm{z}}}

\def\mH{{\bm{H}}}

\def\mW{{\bm{W}}}
\def\mX{{\bm{X}}}
\def\mY{{\bm{Y}}}
\def\mZ{{\bm{Z}}}

\DeclareMathAlphabet{\mathsfit}{\encodingdefault}{\sfdefault}{m}{sl}
\SetMathAlphabet{\mathsfit}{bold}{\encodingdefault}{\sfdefault}{bx}{n}

\newcommand{\E}{\mathbb{E}}

\newcommand{\Var}{\mathrm{Var}}

\def\Algnameabbr{\sas\!\!}

\makeatletter
\newcommand{\printfnsymbol}[1]{%
  \textsuperscript{\@fnsymbol{#1}}%
}
\makeatother

\title{Winner-Take-All Column Row Sampling for Memory Efficient Adaptation of Language Model  }

\author[1]{Zirui Liu\thanks{Equal contribution. The order of authors is determined by flipping a coin.}}
\author[1]{Guanchu Wang\printfnsymbol{1}}
\author[1]{Shaochen Zhong}
\author[1]{Zhaozhuo Xu}
\author[1]{Daochen Zha}
\author[1]{Ruixiang Tang}
\author[2]{Zhimeng Jiang}
\author[1]{Kaixiong Zhou}
\author[3]{Vipin Chaudhary}
\author[3]{Shuai Xu}
\author[1]{Xia Hu}
\affil[1]{Department of Computer Science, Rice University}
\affil[2]{Department of Computer Science, Texas A\&M University}
\affil[3]{Department of Computer and Data Sciences, Case Western Reserve University}
\affil[ ]{{\texttt{\{Zirui.Liu, Guanchu.Wang, Shaochen.Zhong, Zhaozhuo.Xu, Daochen.Zha, Ruixiang.Tang, Kaixiong.Zhou, Xia.Hu\}@rice.edu}, \texttt{zhimengj@tamu.edu},
\affil[ ]\texttt{\{vxc204, sxx214\}@case.edu}}}
\date{}

\begin{document}

\captionsetup[figure]{name={Fig.},labelsep=period} 

\maketitle

\begin{abstract}
With the rapid growth in model size, fine-tuning the large pre-trained language model has become increasingly difficult due to its extensive memory usage. 
Previous works usually focus on reducing the number of trainable parameters in the network. 
While the model parameters do contribute to memory usage, the primary memory bottleneck during training arises from storing feature maps, also known as activations, as they are crucial for gradient calculation. 
Notably, neural networks are usually trained using stochastic gradient descent.
We argue that in stochastic optimization, models can handle noisy gradients as long as the gradient estimator is unbiased with reasonable variance.
Following this motivation, we propose a new family of unbiased estimators called \sas, for matrix production with reduced variance, which only requires storing the sub-sampled activations for calculating the gradient.
Our work provides both theoretical and experimental evidence that, in the context of tuning transformers, our proposed estimators exhibit lower variance compared to existing ones.
By replacing the linear operation with our approximated one in transformers, we can achieve up to 2.7$\times$ peak memory reduction with almost no accuracy drop and enables up to $6.4\times$ larger batch size.
Under the same hardware, \sas enables better down-streaming task performance by applying larger models and/or faster training speed with larger batch sizes.
The code is available at \url{https://github.com/zirui-ray-liu/WTACRS/}.

\end{abstract}

\section{Introduction}
Pre-trained language models (LMs) with transformer architecture have achieved remarkable success in numerous natural language processing (NLP) tasks \citep{vaswani2017attention, devlin2018bert, DBLP:journals/corr/abs-2302-03225, t5, gpt3, yang2023harnessing, zha2023data}.
Specifically,  these models are trained on vast text corpora to acquire general-purpose representations, which are then adapted to a specific task by fine-tuning on task-specific data.
In recent studies, it has been convincingly demonstrated that significantly increasing the number of parameters in pre-trained LMs leads to remarkable improvements in performance~\citep{kaplan2020scaling}. As a result, there is now an urgent necessity to effectively adapt these models, equipped with billion-scale parameters, to a wide range of tasks.

However, a significant disparity exists between the memory requirements of pre-trained LMs and the capacity of current hardware, particularly GPUs.
For example, even a GPU with 24GB memory cannot accommodate the fine-tuning process of the T5-3B model \citep{t5} with batch size one, which boasts three billion parameters.
Without additional techniques, attempting to fine-tune billion-scale LMs on a single GPU is impossible. 
Although model-parallel fine-tuning is feasible, the majority of the time, we cannot bear the expense of acquiring multiple GPUs or the communication overhead involved. To ensure the smooth deployment of language models during the fine-tuning process, it is crucial to adapt them for operation on a single GPU.

To address this issue, 
several parameter-efficient tuning methods are proposed \citep{softprompt, lst, prefix, bitfit, hu2021lora, karimi2021compacter, adapter}.
Specifically,
adapters \citep{adapter, karimi2021compacter} insert a small module into the transformer blocks and only update it while keeping other parameters fixed. Similarly, prompt tuning \citep{softprompt} introduces a small vector that is concatenated with the input embeddings and updated during the tuning process.
LoRA \citep{hu2021lora}  injects trainable rank decomposition matrices into the transformer block, updating them while freezing the others.
Parameter-efficient tuning methods mainly reduce the memory taken by the optimizer states \citep{kingma2014adam, hu2021lora}.
Although the optimizer states contribute to the memory footprint, \emph{storing
activations (or feature maps) is the main memory bottleneck during training} (often $>70\%$) \citep{actnn, checkmate, dtr, tempo}.
Thus, parameter-efficient methods often do not reduce memory usage by much \citep{lst, tempo}.

\begin{wrapfigure}{r}{0.36\textwidth}
 \vspace{-2em}
  \begin{center}
    \includegraphics[width=0.35\textwidth]{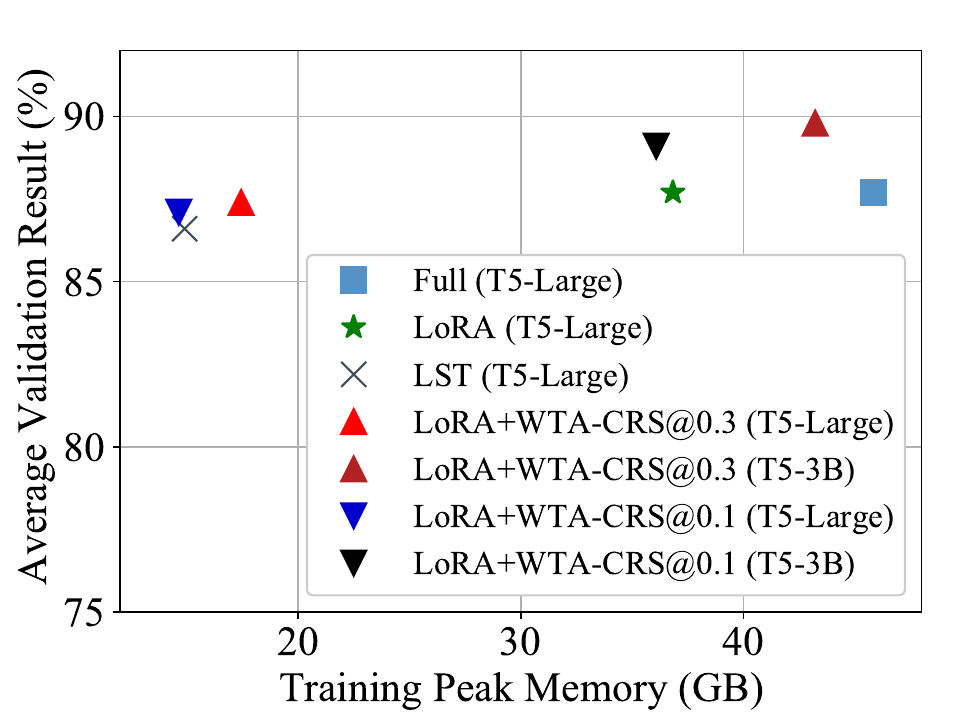}
  \end{center}
    \vspace{-1em}
  \caption{\small Accuracy-memory trade-off of \sas and other memory-efficient tuning methods.
  Unless specially stated, we use the T5-Large in the figure.}
  \label{fig: accuracy-memory_trade-off}
\end{wrapfigure}

In parallel, we can reduce the main memory bottleneck by reducing the activation storage in fine-tuning. Since transformer-based models are mainly built based on the linear layer, a less-explored direction is to replace the expensive matrix multiplication operation with its memory-efficient estimations using column-row sampling (CRS) \citep{adelman2021faster, mcmatmul}.
The key idea of CRS is to sub-sample tensors onto low-dimensional spaces and perform the original operations here. 
Specifically, for the linear operation between two matrices $\mathbf{A}\in\mathbb{R}^{n\times m}$ and $\mathbf{B}\in\mathbb{R}^{m\times q}$ (in the context of machine learning, $\mathbf{A}$ is often activations), 
\textbf{we first sample $k$ ($k<m$) column-row pairs according to a pre-defined distribution.}
Then we obtain $\mathbf{A}'\in\mathbb{R}^{n\times k}$ and $\mathbf{B}'\in\mathbb{R}^{k\times q}$ ($k<m$) by picking $k$ columns of $\mathbf{A}$ and the corresponding rows of $\mathbf{B}$ according to the sampled column-row pairs \citep{mcmatmul}.
Finally, we estimate $\mathbf{A}\mathbf{B}\approx \mathbf{A'}\mathbf{B'}$.
In this way, we only need to store the sub-matrix $\mathbf{A}'$ and $\mathbf{B}'$ in GPU memory to perform the computation.
Moreover, transformer-based models training/tuning are performed with the first-order stochastic optimizer, e.g., Adam \citep{kingma2014adam}.
In stochastic optimization, models can work with noisy gradients, \emph{as long as the gradient estimator is unbiased and has a reasonable variance.}
In view of such, we ask: \textbf{why spend resources on obtaining exact gradients when we are using stochastic optimization?}
Motivated by this, we focus on obtaining unbiased gradients cheaply with approximated matrix multiplication.


The approximation method reduces the memory usage at the cost of giving outputs with variance.
Thus there naturally exists an accuracy-memory trade-off.
The main challenge is how to integrate the approximated matrix multiplication into transformer with minimal gradient variance.
In this paper, we propose a new family of unbiased estimator for matrix multiplication with reduced variance, dubbed Winner-Take-All Column-Row Sampling (\sas).
Compared to CRS, \sas reduces the variance of an estimator by focusing more on high-probability regions of the sampling distribution.
Moreover, \sas can serve as a drop-in replacement for the linear operation in transformers, providing an unbiased weight gradient with reduced memory usage.
As shown in Figure~\ref{fig: accuracy-memory_trade-off}, our method achieves better accuracy-memory trade-off than state-of-the-art memory-efficient tuning methods, e.g., LST \citep{lst} and LoRA \citep{hu2021lora}. 
Moreover, since \sas executed at the operation level, it is orthogonal to most of the existing parameter-efficient tuning methods.
Our contributions are highlighted as follows:

\begin{itemize}[nosep]
    \item We design a new family of unbiased estimator for matrix multiplication with reduced variance. 
    We theoretically and experimentally verify that it has smaller variance than the established one under the context of tuning transformer.
    \item By replacing the linear operation with \sas in transformers, we can achieve up to 2.7$\times$ peak memory reduction with almost no accuracy drop, and enables up to $6.4\times$ larger batch size.
    As shown in Figure \ref{fig: accuracy-memory_trade-off}, \sas stands out as an exceptional solution capable of fine-tuning T5-3B using a mere 40GB GPU memory budget, with three billion parameters. 
    Thus, we achieve remarkable advancements in the adaptation of LMs on downstream tasks. 
    \item We implement \sas as a ready-to-use extension for Pytorch with an easy-to-use API that can also be combined with other memory-saving techniques.
\end{itemize}



\section{Background and Preliminary}
\label{sec: background}


In this section, we first analyze the memory usage of transformers.
Then we introduce the background on the approximated matrix multiplication.

\subsection{The Memory Usage of Transformers}
\label{sec: memory}

\begin{wrapfigure}{r}{0.35\textwidth}
 \vspace{-2.5em}
  \begin{center}
    \includegraphics[width=0.35\textwidth]{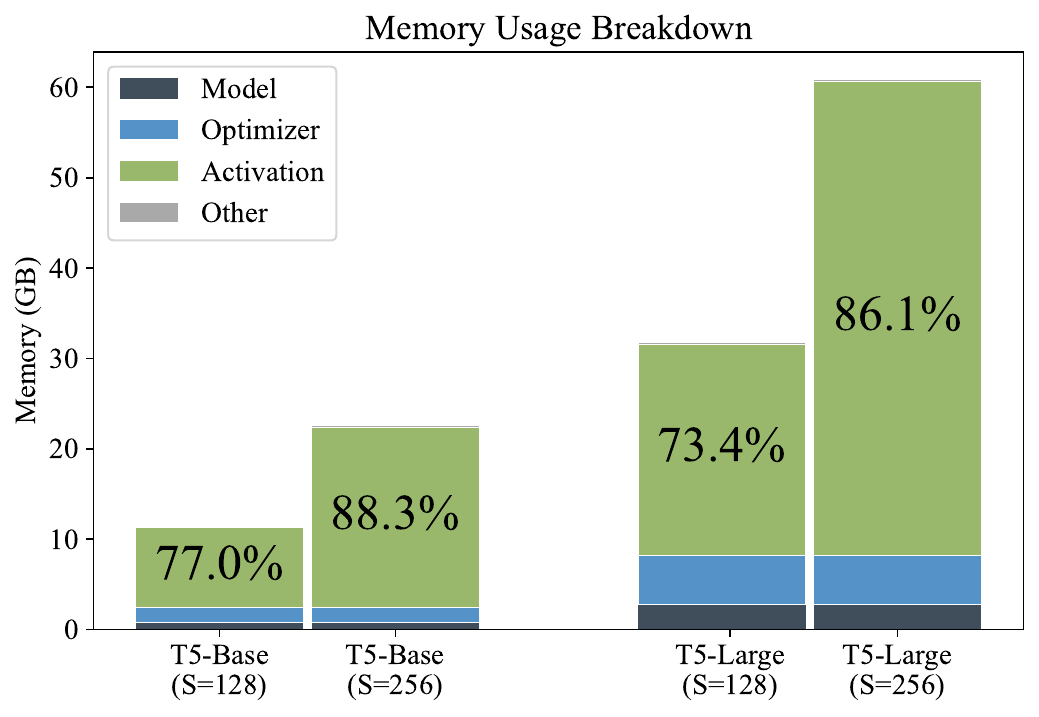}
  \end{center}
    \vspace{-1em}
  \caption{The GPU memory usage breakdown for fine-tuning T5 \citep{glue}, where the batch size $B$ is 64 and sequential length $S$ is 128 or 256.}
  \label{fig: percentage}
  \vspace{-3em}
\end{wrapfigure}

In each training step of backpropagation, it has exactly two phases, i.e., one forward phase and one backward phase. 
Transformer-based models are mainly built based on the linear operation,
which can be written as:

\begin{subequations}
\begin{align}
        \texttt{Forward Pass}~~~~~~~~\mZ &= \mm({\color{red}\mH}, \mW),
        \label{eq: fwd}\\
    \texttt{Backward Pass}~~~~~~~~\nabla\mH &= \mm(\nabla\mZ, \mW^\top), \label{eq: bwd_x}\\
    \nabla\mW &= \mm({\color{red}\mH^\top}, \nabla\mZ),
    \label{eq: bwd_w}
\end{align}
\end{subequations}

where $\mm(\cdot, \cdot)$ is the General Matrix Multiplication operation, $\mH$ and $\mZ$ are the activation (or input feature maps) and output feature maps, respectively.
$\mW$ is the weight of the linear layer.
$\nabla\mH$, $\nabla\mW$, and $\nabla\mZ$ are the gradient of $\mH$, $\mW$, and $\mZ$, respectively.
From \Eqref{eq: bwd_w}, activations $\mH$ are used in the backward phase.
In commonly used deep learning framework \citep{abadi2016tensorflow, paszke2019pytorch}, it requires storing $\mH$ in GPU memory during the forward pass, for calculating the weight gradient $\nabla\mW$ in the backward pass.

Previous works show that although the model parameters
contribute to the memory footprint, 
activations (e.g., storing $\mH$) are the main memory bottleneck during training \citep{actnn, checkmate, dtr, tempo}.
To get a sense of the scale, we show in Figure \ref{fig: percentage} that for popular transformer models like T5, activations may take roughly $73\sim88\%$ of the total memory, depending on the batch size $B$ and sequential length $S$.


\subsection{Approximated \texorpdfstring{$\mm$}{~} With Sampling}
\label{sec: prelim}

Let $\mX\in\mathbb{R}^{n\times m}$,  $\mY\in\mathbb{R}^{m\times q}$ be two matrices.
The goal is to efficiently estimate the matrix production $\mX\mY$.
Singular Value Decomposition (SVD) outputs provably optimal low-rank estimation of $\mX\mY$ \citep{adelman2021faster}.
However, SVD is almost as expensive as matrix production itself.
Instead, the sampling algorithm is proposed to approximate the matrix product $\mX\mY$
by sampling $k$ columns of $\mX$ and corresponding rows of $\mY$ to form smaller matrices, which are then multiplied as usual \citep{mcmatmul, drineas2001fast}:
\begin{align}
\mm(\mX, \mY) &= \sum_{i=1}^{m} \mX_{:,i} \mY_{i,:}\approx \sum_{t=1}^{k} \frac{1}{kp_{i_t}}\mX_{:,i_t}\mY_{i_t,:} =\mX'\mY',
        \label{eq: approx_matmul}
\end{align}

where $\mX_{:,i}\in\mathbb{R}^{n\times 1}$ and $\mY_{i,:}\in\mathbb{R}^{1\times q}$ are the $i^{\mathrm{th}}$ column and row of $\mX$ and $\mY$, respectively.
\textbf{In this paper,  we call $(\mX_{:,i}, \mY_{i,:})$ the $i^{\mathrm{th}}$ column-row pair.}
$k$ is the number of sampled pairs ($1\leq k\leq m$).
$\mathcal{P}=\{p_i\}_{i=1}^m$ is a probability distribution over the column-row pairs.
$i_t\in\{1,\cdots m\}$ is the index of the sampled column-row pair at the $t^{\mathrm{th}}$ trial.
$s_{t}$ is the scale factor.
$\mX'\in\mathbb{R}^{n\times k}$ and $\mY'\in\mathbb{R}^{k \times q}$ are the normalized sub-matrices sliced according to the sampled column-row pairs.


Existing work \citep{mcmatmul} shows $\mX'\mY'$ is an unbiased estimation of $\mX\mY$, i.e., $\E[\mX'\mY']=\mX\mY$.
Furthermore, the approximation error $\E[||\mX\mY-\mX'\mY'||_F]$ is minimized when the probabilities
$\{p_i\}_{i=1}^m$ are proportional to the product of the column-row Euclidean norms \citep{mcmatmul} (Proof in Appendix \ref{app: theory}):
\begin{equation}
\label{eq: col_row_norm}
    p_i = \frac{ ||\mX_{:,i}||_2\ ||\mY_{i,:}||_2}{\sum_{j=1}^{m} ||\mX_{:,j}||_2\ ||\mY_{j,:}||_2}.
\end{equation}

As we analyzed in Section \ref{sec: memory}, storing the activation $\mH$ is the major memory bottleneck.
\textbf{
If we can replace $\mm(\mH^\top, \nabla\mZ)$ in \Eqref{eq: bwd_w} with $\mH'^\top\nabla\mZ'$ following the paradigm of \Eqref{eq: approx_matmul}, 
 then we only need $\mH'$ instead of $\mH$ in GPU memory to compute the gradient, which significantly decreases the memory usage of activations.}
 This estimation linearly reduces the memory complexity from $\mathcal{O}(nm)$ to $\mathcal{O}(nk)$.
 Also, the total number of floating point operations (FLOPs) is reduced as well since the computation is executed on two smaller matrices.
 \emph{For the ease of illustration, in this paper we call the distribution in \Eqref{eq: col_row_norm} the \textbf{column-row index distribution}}.
In the next section, 
we explore how to reduce memory usage via sampling-based matrix multiplication.

 


\section{Methodology}

In recent years, we have observed that deep neural network training can be performed almost entirely with first-order \emph{stochastic optimization} \citep{kingma2014adam}.
Thus intuitively, \emph{in stochastic optimization we can reduce the resources spent on obtaining gradients, as long as the estimated gradient is unbiased with reasonable variance} \citep{chmielminimum, chmiel2021logarithmic, oktay2020randomized}.
Following this motivation, we first design a new unbiased estimator for matrix multiplication with reduced variance compared to the one in \Eqref{eq: approx_matmul} (Section \ref{sec: sum_and_sample} ).
Then we introduce how to replace the \mm in Transformer with its approximated version to reduce the memory usage (Section \ref{sec: where}).

\subsection{Winner-Take-All Column-Row Sampling: A New Unbiased Estimator for \texorpdfstring{$\mm$}{~}}
\label{sec: sum_and_sample}
In this section, we mathematically design a new unbiased estimator for \mm with reduced variance called WTA-CRS (Winner-Take-All Column-Row Sampling).
Following the notation in Section \ref{sec: prelim}, let $\mX\in\mathbb{R}^{n\times m}$,  $\mY\in\mathbb{R}^{m\times q}$ be two matrices.
$\mathcal{P}=\{p_i\}_{i=1}^m$ is the column-row index distribution in \Eqref{eq: col_row_norm}\footnote{Here we note that the theoretical analysis in this section can be applied to any probability distribution, not only limited to the one in \Eqref{eq: col_row_norm}.}.
We first define the variable $f(i)$ as 
\begin{equation}
f(i)=\frac{\mX_{:i}\mY_{i:}}{p_i}, \nonumber
\end{equation} 

$f(i)$ is an unbiased estimation for the matrix production between $\mX$ and $\mY$.
To see this, 
\begin{align}
\E_{j\sim \mathcal{P}}[f(j)]=\sum_{i=1}^m  p_{i}\frac{\mX_{:,i}\mY_{i:}}{p_{i}}=\mX\mY.   \nonumber 
\end{align}
We note that the prior approximated matrix multiplication in \Eqref{eq: approx_matmul} is the direct extension of $f(i)$ by taking the average of $\{f(i_t)\}_{t=1}^k$ among $k$ independent random trials to reduce the variance. 
Here we explore an alternative approach to reduce the variance of $f(i)$ beyond simple averaging. Our core idea is to partition the column-row index distribution $\mathcal{P}=\{p_i\}_{i=1}^m$ into two complementary regions based on the probability mass: a high-probability region $\mathcal{P}^\mathcal{C}$ and a low-probability region $\mathcal{P}^{\mathcal{D}\backslash\mathcal{C}}$, where $\mathcal{D}=\{1,\cdots, m\}$ is the whole set and $\mathcal{C}$ is the set of the column-row index with the largest probability.
\textbf{Let $\mathcal{C}$ be the set of column-row pair indices associated with $|\mathcal{C}|$ largest $p_i$.}
We define \sas estimator for $\mX\mY$ as follows:
\begin{align}
\E_{j\sim \mathcal{P}^{\mathcal{D}\backslash\mathcal{C}}} \Big[\sum_{c\in\mathcal{C}} f(c)p_c + (1 - \sum_{c\in\mathcal{C}}p_c)f(j)\Big].
\label{eq: sum_and_sample}
\end{align}

We note that \textbf{the random variable in \Eqref{eq: sum_and_sample} is the column-row pair index $j$, and is only sampled from $\mathcal{D}\backslash\mathcal{C}$}.
The estimator defined in \Eqref{eq: sum_and_sample} contains two parts. 
The first part 
$\sum_{c\in\mathcal{C}} f(c)p_c$ has no relationship with the random variable $j$ and is summed deterministically. 
The second part $f(j)$ is sampled stocastically, but scaled by the factor $(1 - \sum_{c\in\mathcal{C}} p_c)$.
When $\mathcal{P}=\{p_i\}_{i=1}^m$  is concentrated on a small number of
atoms, the scaling factor $(1 - \sum_{c\in\mathcal{C}}p_c)$ for the stochastic term  should be small.
Therefore, we intuitively expect the estimator to have a small variance in this case due to a small scaling factor.
In this way, we reduce the variance of an estimator by focusing more on high-probability regions of the distribution (winner-take-all).
Below we formalize this intuition by showing the statistical property of our estimator regarding the bias and variance, respectively.

\begin{restatable}[Proof in Appendix~\ref{app: theory_unbias}]{theo}{theounbias}
\label{theo: unbias}
The estimator defined in \Eqref{eq: sum_and_sample} is an unbiased estimator for matrix production $\mX\mY$, i.e, $\E_{j\sim \mathcal{P}^{\mathcal{D}\backslash\mathcal{C}}} [\sum_{c\in\mathcal{C}} f(c)p_c + (1 - \sum_{c\in\mathcal{C}}p_c)f(j)]=\mX\mY$.
\end{restatable}

Theorem \ref{theo: unbias} states that our proposed estimator in \Eqref{eq: sum_and_sample} is unbiased.
Below we compare our proposed estimator to the CRS estimator in \Eqref{eq: approx_matmul} in terms of the variance.
Suppose we have the budget of only utilizing $k$ column-row pairs for approximating the matrix production.
From the implementation perspective, the estimator defined in \Eqref{eq: approx_matmul} estimates $\mm(\mX,\mY)$ as:
\begin{align}
\texttt{(CRS)}~~~~~~g(\mX,\mY) = \frac{1}{k}\sum_{t=1}^k f(i_t), ~~~i_1,\cdots, i_k \overset{\text{i.i.d}}{\sim} \mathcal{P}.
\label{eq: crs_implementation}
\end{align}

Our estimator defined in \Eqref{eq: sum_and_sample} splits the budget $k$ into two parts. Namely, the first part explicitly sums the expectation terms for the largest probability group $\mathcal{C}$ ($|\mathcal{C}| < k$),
while stochastically average $k-|\mathcal{C}|$ samples drawn from $\mathcal{D}\backslash\mathcal{C}$ to estimate the remaining terms, up to scale:
\begin{align}
(\texttt{WTA-CRS})~~~
\hat g(\mX,\mY) = \sum_{c\in\mathcal{C}} f(c)p(c)+\frac{1 - \sum_{c\in\mathcal{C}}p_c}{k-|\mathcal{C}|}\sum_{j=1}^{k-|\mathcal{C}|}f(j), ~~~i_1,\cdots, i_{k-|\mathcal{C}|} \overset{\text{i.i.d}}{\sim} \mathcal{P}^{\mathcal{D}\backslash\mathcal{C}}.
\label{eq: our_estimator}
\end{align} 



\begin{restatable}[Proof in Appendix~\ref{app: theory_var}]{theo}{theounvar}
\label{theo: var}

Suppose the total budget of column-row pairs is $k$. If $\mathcal{C}$ satisfies 
\begin{equation}
\label{eq: var_thresh}
\sum_{c\in\mathcal{C}}p_c > \frac{|\mathcal{C}|}{k},
\end{equation}
then we have $\Var[\hat g(\mX,\mY)] < \Var[g(\mX,\mY)]$.
Moreover, $\Var[\hat g(\mX,\mY)]$ is minimized when $|\mathcal{C}|=\min_{|\mathcal{C}|\in\{0,\cdots,k\}}\frac{1 - \sum_{c\in\mathcal{C}}p_c}{k-|\mathcal{C}|}$.
\end{restatable}

Both the left- and right-hand sides of \Eqref{eq: var_thresh} depend on the size of the highest probability group $|\mathcal{C}|$, which controls the number of high probability column-row pairs that are directly added without sampling.
Below we experimentally investigate whether \Eqref{eq: var_thresh} holds under the context of fine-tuning the transformer-based model with varying $|\mathcal{C}|$.

\begin{figure*}[h!]
    \centering
    \begin{subfigure}[h]{0.34\linewidth}
    \includegraphics[width=\linewidth]{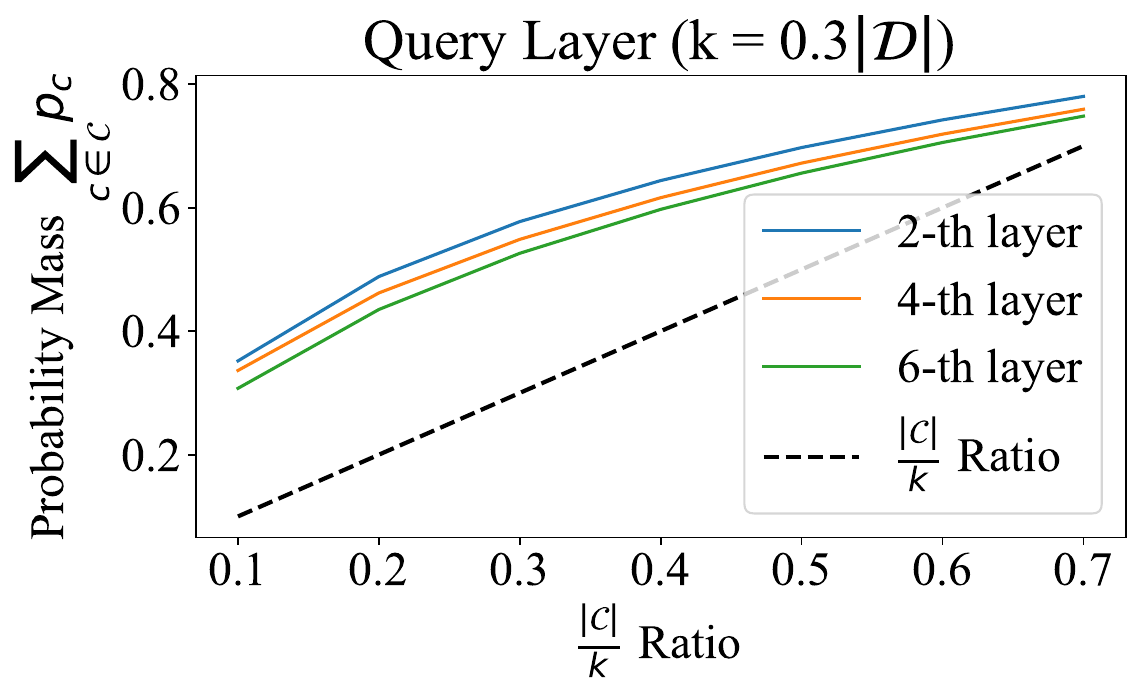}
    \end{subfigure}
    \begin{subfigure}[h]{0.31\linewidth}
    \includegraphics[width=\linewidth]{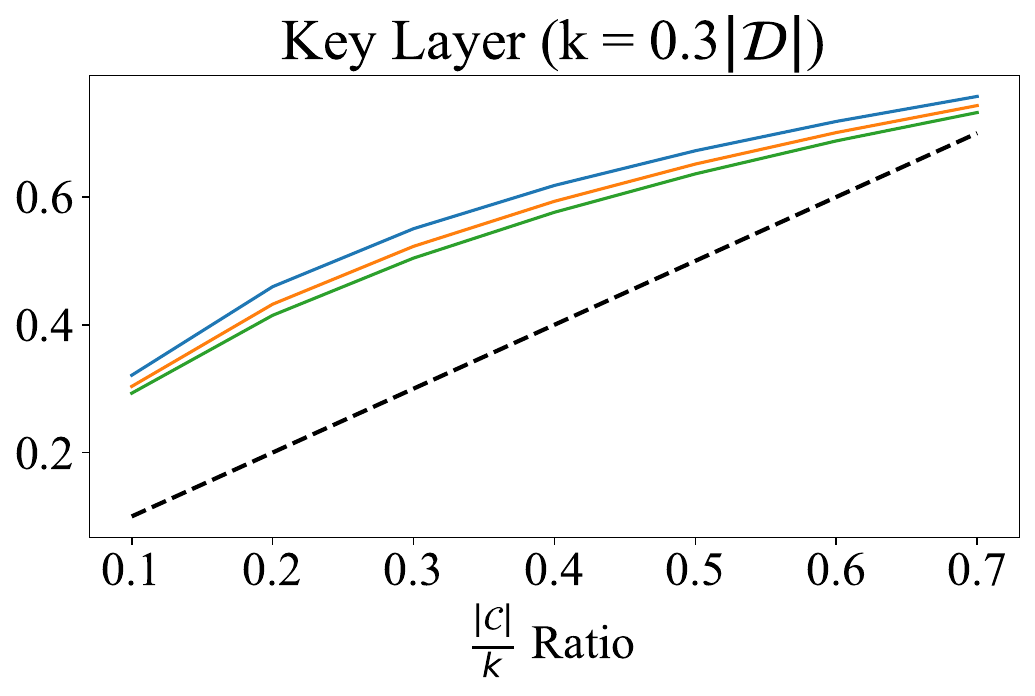}
    \end{subfigure}
    \begin{subfigure}[h]{0.31\linewidth}
    \includegraphics[width=\linewidth]{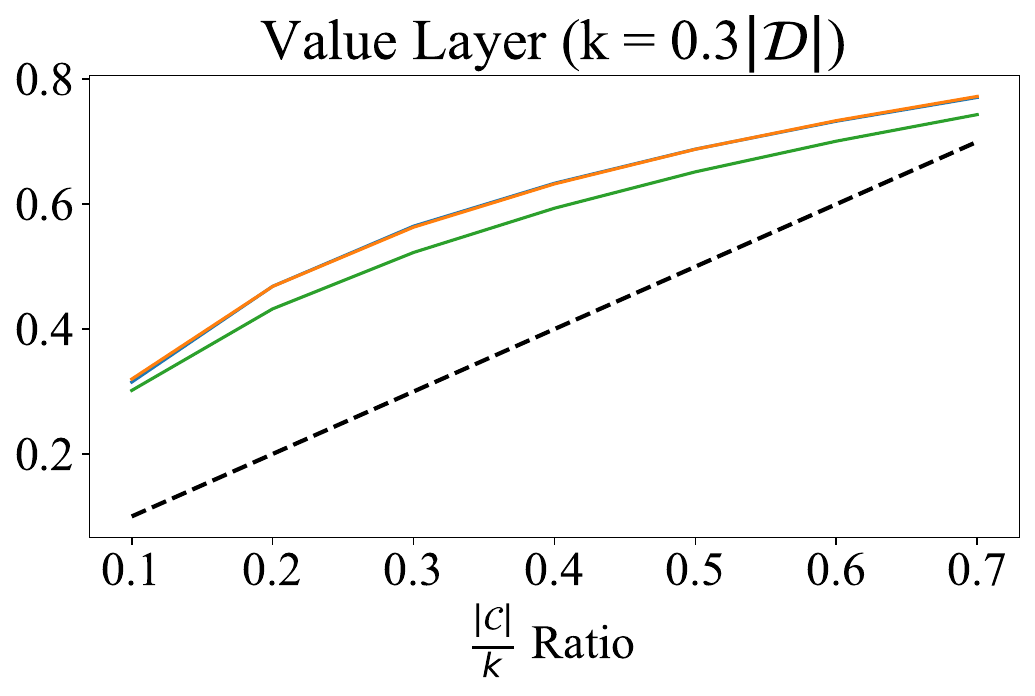}
    \end{subfigure}
        \vspace{-.5em}
    \caption{The probability mass $\sum_{c\in\mathcal{C}}p_c$ versus $\frac{|\mathcal{C}|}{k}$ in \Eqref{eq: var_thresh} at $k=0.3|\mathcal{D}|$.
    Here we visualize the column-row index distribution of query/key/value projection layer in the T5-base model, which is fine-tuned on RTE dataset.
    More similar results can be found in Appendix \ref{app: more_res_theo2}.}
    \vspace{-.5em}
    \label{fig: exp_assump_analysis_th2_direct}
\end{figure*}

\textbf{Experimental analysis.}
As shown in Figure \ref{fig: exp_assump_analysis_th2_direct},
we visualize the two terms in \Eqref{eq: col_row_norm} for the column-row index distribution of query, key, and value projection in the self-attention module, respectively \citep{vaswani2017attention}.
Specifically, we fix the total column-row pair budget $k=0.3|\mathcal{D}|$  and change the size of the highest probability group $|\mathcal{C}|$ from 0 to $k$.
We conclude that \Eqref{eq: var_thresh} holds for most of the layers when fine-tuning transformers.
Thus, we expect our \sas has better performance than CRS for adapting transformer-based models, which is later experimentally verified in Section \ref{sec: exp}.

\subsection{Compress \texorpdfstring{$\mm$}{~} in Transformers with \texorpdfstring{$\sas$}{~}}
\label{sec: where}

Previous work has shown that unbiasedness of the estimated gradient is crucial for the proper convergence of stochastic gradient descent \citep{chmielminimum, chmiel2021logarithmic, actnn, gact}.
As shown in Section \ref{sec: memory}, we have three \mm in the linear layer.
Below we investigate how to replace \mm with its approximated version in a way that the estimated gradient is unbiased.

\noindent
\textbf{Unbiasedness.}
Previous work has shown that to ensure the unbiasedness of the gradient, the approximation can only be applied during the backward pass \citep{actnn, liu2022rsc, adelman2021faster}.
The rationale behind this conclusion is that we have $\E[f(x)]\neq f(\E[x])$ for any non-linear function $f(\cdot)$, e.g., $\E[x^2]\neq \E^2[x]$.
Thus if we replace the forward \mm in \Eqref{eq: fwd}, even when the approximation method gives an unbiased estimation, i.e., $\E[\hat{g}(\mH, \mW)]=\mH\mW=\mZ$, 
the output activations (e.g., $\gelu(\mZ)$) are still biased since the activation function is non-linear, namely,
\begin{equation}
    \gelu(\hat g(\mH,\mW))=\gelu(\E[\mZ])\neq \E[\gelu(\mZ)].\nonumber
\end{equation}

To ensure the unbiasness of gradient and reduce the memory usage of storing $\mH$, as shown in the example of Figure \ref{fig: deploy_linear}, \textbf{we only replace \mm in the backward pass with its approximation (e.g., \Eqref{eq: bwd_w}), while leaving the forward one unchanged (e.g., \Eqref{eq: fwd}).}
\textbf{We show in Appendix \ref{app: unbias_grad} that the estimated weight gradient is unbiased in this case.}

\begin{figure}[t!]
    \centering
    \includegraphics[width=0.9\linewidth]{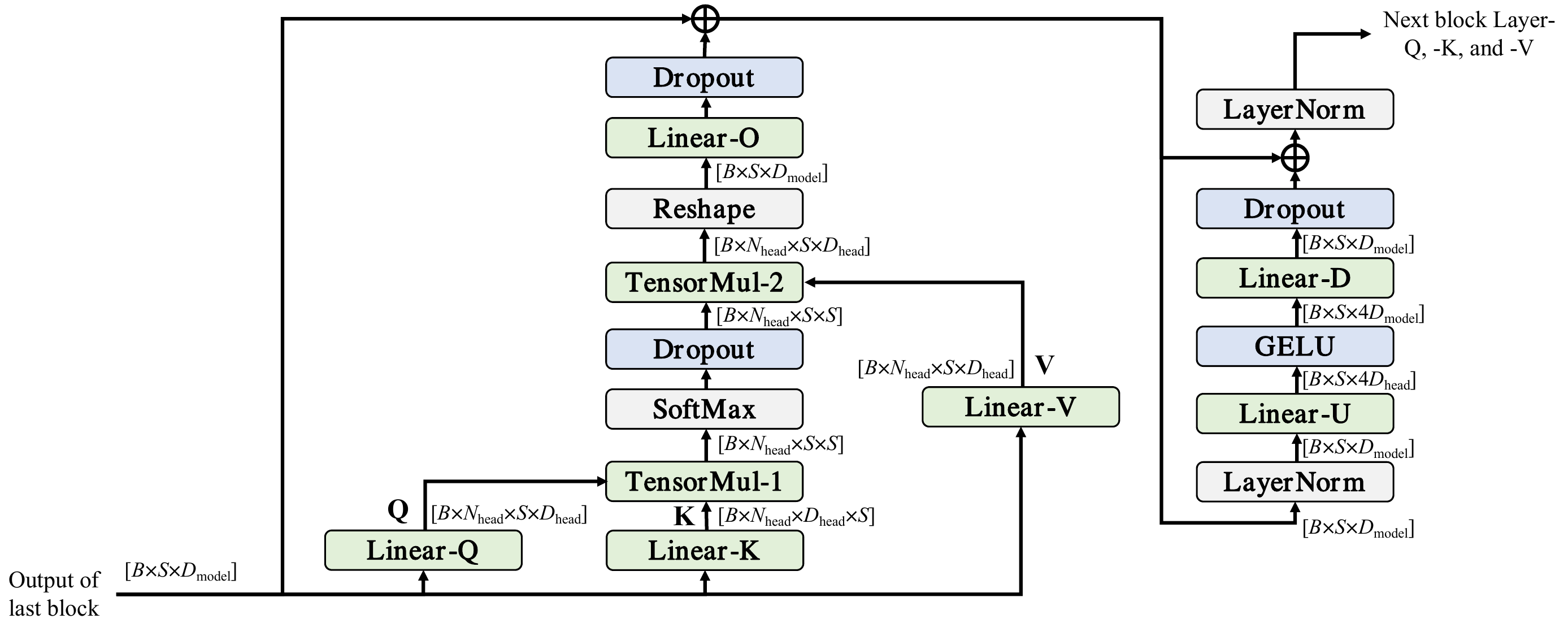}
    \caption{
    The diagram of a single Transformer block. 
        The shape of activations is annotated, where $B, S, D_{\text{model}}$,  $N_{\text{head}}$, and $D_{\text{head}}$ are the batch size, sequence length, hidden size, number of attention heads, and head dimension, respectively. \Algnameabbr{} can be applied to the operators in green; the activation maps of operators in blue can be losslessly compressed; and those in gray are not compressed in this paper. The idea of this figure is inspired by \citep{andoorveedu2022tempo}.}
    \label{fig:model_config}
        \vspace{-1em}
\end{figure}

\begin{wrapfigure}{r}{0.5\textwidth}
 \vspace{-2em}
  \begin{center}
    \includegraphics[width=0.45\textwidth]{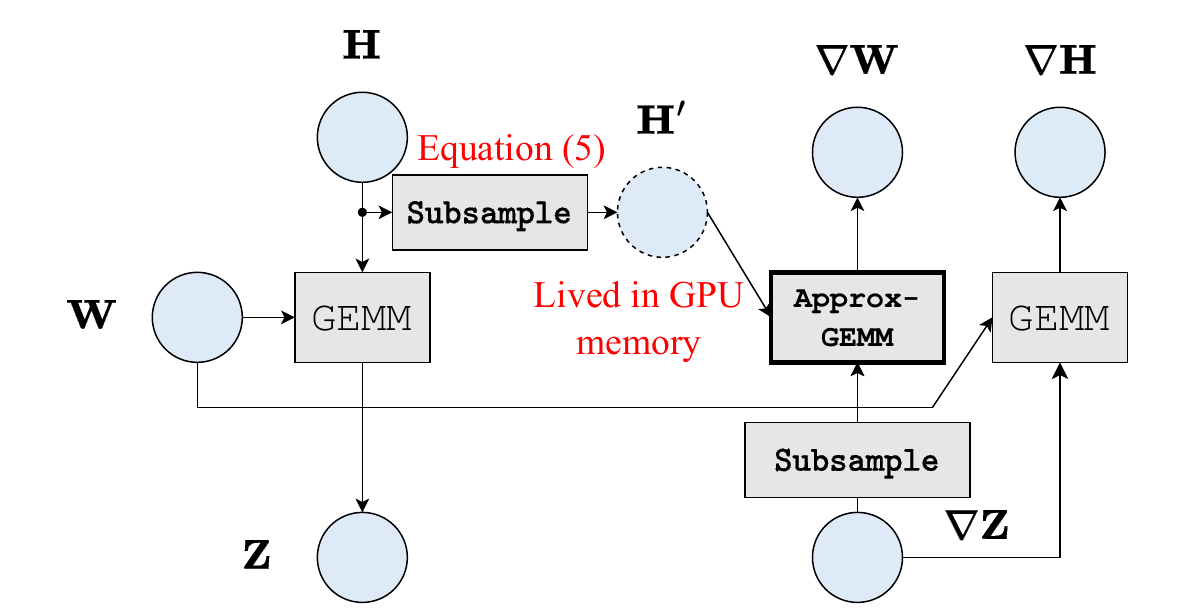}
  \end{center}
    \vspace{-1em}
  \caption{The illustration of how to deploy \sas to linear layers.
  We only replace \mm in \Eqref{eq: bwd_w} with its approximated version using \sas.
  The pseudocode is given in Appendix \ref{app: implementation} Algorithm \ref{algo: approx_q_linear}.}
  \vspace{-1em}
  \label{fig: deploy_linear}
\end{wrapfigure}

\paragraph{Implementation.}
Here we present how we implement \sas in \Eqref{eq: our_estimator} in practice.
For the linear layer, as we analyzed, we only replace \mm in \Eqref{eq: bwd_w} with its approximated version.
In this case, $\mX$ and $\mY$ in \Eqref{eq: our_estimator} are activation $\mH^\top$ and output gradient $\nabla\mZ$, respectively.
Given the total column-row pair budget $k$, the \textbf{first} step is to build the deterministic index set $\mathcal{C}$, where each element is summed explicitly without sampling.
Note that $\mathcal{C}$ is a set of indices with the highest probabilities in \Eqref{eq: col_row_norm}. Thus, to build $\mathcal{C}$, we only need to determine its size, denoted as $|\mathcal{C}|$, which minimizes the variance of the estimator.
As Theorem \ref{theo: var} suggested, we set $|\mathcal{C}|=\min_{|\mathcal{C}|\in\{0,\cdots,k\}}\frac{1 - \sum_{c\in\mathcal{C}}p_c}{k-|\mathcal{C}|}$.
The \textbf{second} step is to sample $k-|\mathcal{C}|$ column-row indices from the remaining distribution $\mathcal{P}^{\mathcal{D}\backslash\mathcal{C}}$ to obtain the set $\mathcal{C}_\text{stoc}$, where $|\mathcal{C}_\text{stoc}|=k-|\mathcal{C}|$.
The \textbf{third} step is to build sub-sampled $\mH'$ with only rows from $\mathcal{C}\cup\mathcal{C}_\text{stoc}$.
Note that for rows in $\mH'$ from $\mathcal{C}_\text{stoc}$, we need to normalize it by $\frac{1 - \sum_{c\in\mathcal{C}}p_c}{k-|\mathcal{C}|}$ according to \Eqref{eq: our_estimator}.
We illustrate the above process in Figure \ref{fig: deploy_linear}.
The pseudocode to Appendix \ref{app: implementation} Algorithm \ref{algo: approx_q_linear}.
We leverage Liger kernel \citep{hsu2024ligerkernelefficienttriton} to further speedup our fine-tuning pipeline.

\paragraph{Scope.}
Here we show which operation can be replaced with its approximation version.
As shown in Figure \ref{fig:model_config}, the transformer is mainly consisted of linear layer, TensorMul, and other operations (e.g., GeLU, Dropout, LayerNorm).
TensorMul in Figure \ref{fig:model_config} refers to the multiplication between two four-dimensional tensors.
Our \Algnameabbr{} \emph{can be applied to Linear-Q, -K, -V, -O, -U, -D, TensorMul-1, and TensorMul-2~(in green).}
The activations of Dropout and GELU operations~(in blue) can be losslessly compressed.
The Softmax and LayerNorm operators~(in gray) remain unchanged.


\section{Related Work and Discussion}
Due to the page limit, we discuss the related work on approximated matrix multiplication and activation compression.  
Other related topics,
e.g., parameter-efficient fine-tuning and gradient checkpointing, can be found in
Appendix \ref{app: related_work}. 
We also discuss the limitation and potential negative social impact in Appendix \ref{app: related_work}.

\textbf{Approximated Matrix Multiplication.}
In the context of neural networks, approximated matrix multiplication methods can be broadly categorized into two main groups:
(1) Butterfly-based methods \citep{chen2021pixelated, dao2022monarch} replace dense weight matrices with butterfly matrices.
We note that they focus on the weight matrix and are orthogonal to our research, as we concentrate on sub-sampling the activation matrix.
(2) Column-row sampling (CRS) methods \citep{mcmatmul, adelman2021faster, liu2022rsc}  select important rows and columns from the input matrices and perform the multiplication on the sampled matrix. 
Our work is closely aligned with this second research line.
\citep{adelman2021faster, liu2022rsc} share similarities with our research in terms of utilizing CRS for approximating matrix multiplication within neural networks. 
The main distinction lies in how to select the column-row pairs. 
Specifically, \citep{adelman2021faster} deterministically selects column-row pairs without scaling, whereas our estimator divides the column-row pairs into a deterministic component and a stochastic component.
As we analyzed, selecting column-row pairs deterministically is biased.
Later we show that this approach may cause a significant accuracy drop (\textbf{``Deterministic'' in Figure \ref{fig:ablation_exp}}).

\textbf{Activation Quantization.}
The activation quantization methods focus on quantizing the activation into low numerical precision numbers, e.g., 8-bit integers \citep{actnn, gact, exact, wang2022towards, wang2022bed}.
Here we discuss the difference between these two works in terms of compression ratios.
According to Table 5 in \citep{gact}, when GACT is combined with Swapping, i.e., offloading quantized activation to main memory, it achieves a peak memory usage compression rate of $1.73\times$ for Bert-Large.
Our work also compresses the activation, but in a different way.
\emph{We emphasize that our work is orthogonal to activation quantization in the sense that our work essentially reduces the dimension of activation.}
This distinction allows our method to be readily combined with activation quantization techniques, offering the potential for even more aggressive compression.

\section{Experiments}
\label{sec: exp}

In this section, we design experiments to answer the following research questions: 
\textbf{RQ1:} How effective is \Algnameabbr{} in terms of accuracy with reduced memory usage?
\textbf{RQ2:} How sensitive is \Algnameabbr{} affected by its key hyper-parameters?
\textbf{RQ3:} 
\sas contains two parts, i.e., the deterministic summation part and the statistical sampling part. Are they both necessary?
\textbf{RQ4:} How is the fine-tuning speed affected by \sas?

\subsection{Experiment Setup}

\paragraph{Datasets and Evaluation Protocol.} 
Following most of the previous work, we adopt GLUE benchmark~\citep{wang2018glue} to evaluate the effectiveness of different methods, including the CoLA, SST-2, MRPC, QQP, MNLI, QNLI, RTE, and STS-B datasets.
For the SST-2, MNLI, QNLI, and RTE datasets, we report the validation accuracy. For CoLA, we use Matthew's correlation as the evaluation metric. The F1 score is reported for both MRPC and QQP tasks, while the Pearson-Spearman correlation is used to evaluate the performance on the STS-B dataset.
To evaluate the memory usage, we report the peak GPU memory usage and compression rate during the fine-tuning process with Huggingface API \citep{huggingface}.



\textbf{Compared Methods and Adopted Models.} 
We consider three methods to compare in this paper: Full fine-tuning~(Full), LoRA~\citep{hu2021lora}, and Ladder Side-tuning~(LST)~\citep{lst}. Specifically, \textbf{Full} tunes all of the parameters in the model to provide an upper bound of accuracy; \textbf{LoRA} inserts trainable low-rank matrices into the model to parameterize the
weights’ changes; \textbf{LST} injects a trainable ladder side structure.
Since \Algnameabbr{} essentially replace the linear operation with approximated one, \textbf{we emphasize that our \Algnameabbr{} is compatible with all these three compared methods, i.e., they can be combined together towards smaller memory usage.}
For the backbone model, we follow the previous work~\citep{lst, adapter, hu2021lora} to adopt the Bert-Base \citep{devlin2018bert}, Bert-Large, T5-Base, T5-Large, and T5-3B~\citep{raffel2020exploring} for evaluating the effectiveness of different methods.

\textbf{Hyperparameter Settings.}  
For \Algnameabbr{}, it only has one hyperparameter $k$, which controls the column-row pair budget. 
We assign the same $k$ to all replaceable linear operations in the model.
We consider the normalized column-row pair budget $k/|\mathcal{D}| \in \{0.3, 0.1\}$, which are denoted as \Algnameabbr{}@0.3 and \Algnameabbr{}@0.1, respectively.
We also consider the combination of \Algnameabbr{} and LoRA to further reduce the memory cost of both optimizer and activations.
The detailed hyperparameters are given in Appendix~\ref{appendix:hyper_setting}.
All reported results are averaged over three random trials.


\begin{table}[t!]
    \centering
    \captionsetup{skip=5pt} 
    \caption{The GLUE benchmark results with T5 and Bert at different scales.}
    \label{tab:glue_eval}
\resizebox{\textwidth}{!}{
\begin{tabular}{llccccccccc}
\hline
Model                       & Method           & CoLA          & SST-2         & MRPC          & QQP           & MNLI          & QNLI          & RTE           & STS-B         & AVG  \\ \hline
\multirow{4}{*}{BERT-Base}  & Full             & 60.9$\pm$\scriptsize{1.89} & 92.2$\pm$\scriptsize{0.34 }& 87.9$\pm$\scriptsize{0.46} & 87.8$\pm$\scriptsize{0.01} & 83.7$\pm$\scriptsize{0.05} & 90.7$\pm$\scriptsize{0.14} & 66.4$\pm$\scriptsize{0.36} & 88.1$\pm$\scriptsize{0.27} & 82.2 \\
                            & LoRA             & 61.6$\pm$\scriptsize{0.25 }& 91.7$\pm$\scriptsize{0.17 }& 90.0$\pm$\scriptsize{0.34} & 86.9$\pm$\scriptsize{0.1}  & 83.6$\pm$\scriptsize{0.02} & 90.8$\pm$\scriptsize{0.17} & 68.2$\pm$\scriptsize{0.36} & 87.6$\pm$\scriptsize{0.52} & 82.6 \\
                            & WTA-CRS@0.3      & 60.7$\pm$\scriptsize{0.89} & 90.2$\pm$\scriptsize{0.06} & 87.0$\pm$\scriptsize{0.14} & 87.5$\pm$\scriptsize{0.03} & 83.4$\pm$\scriptsize{0.06} & 90.4$\pm$\scriptsize{0.07} & 65.9$\pm$\scriptsize{0.18} & 89.3$\pm$\scriptsize{0.2}  & 81.8 \\
                            & LoRA+WTA-CRS@0.3 & 61.5$\pm$\scriptsize{1.08} & 89.6$\pm$\scriptsize{0.52 }& 89.6$\pm$\scriptsize{0.09} & 86.3$\pm$\scriptsize{0.02} & 82.8$\pm$\scriptsize{0.35} & 90.6$\pm$\scriptsize{0.16} & 67.9$\pm$\scriptsize{0.72} & 87.3$\pm$\scriptsize{0.7}  & 81.9 \\ \hline
\multirow{5}{*}{T5-Base}    & Full             & 60.1$\pm$\scriptsize{0.37} & 94.9$\pm$\scriptsize{0.29 }& 91.5$\pm$\scriptsize{0.29} & 88.5$\pm$\scriptsize{0.07} & 87.0$\pm$\scriptsize{0.1}  & 93.3$\pm$\scriptsize{0.03} & 79.4$\pm$\scriptsize{0.78} & 90.6$\pm$\scriptsize{0.14} & 85.7 \\
                            & LoRA             & 60.6$\pm$\scriptsize{0.94} & 60.6$\pm$\scriptsize{0.94 }& 92.2$\pm$\scriptsize{0.31} & 87.4$\pm$\scriptsize{0.06} & 86.2$\pm$\scriptsize{0.06} & 93.4$\pm$\scriptsize{0.03} & 80.6$\pm$\scriptsize{0.74} & 90.7$\pm$\scriptsize{0.05} & 85.7 \\
                            & LST              & 55.5$\pm$\scriptsize{0.24} & 94.0$\pm$\scriptsize{0.17 }& 91.1$\pm$\scriptsize{0.18} & 87.4$\pm$\scriptsize{0.03} & 85.7$\pm$\scriptsize{0.13} & 93.4$\pm$\scriptsize{0.0 } & 72.7$\pm$\scriptsize{0.54} & 90.4$\pm$\scriptsize{0.06} & 83.8 \\
                            & WTA-CRS@0.3      & 60.9$\pm$\scriptsize{0.52} & 94.8$\pm$\scriptsize{0.14 }& 91.1$\pm$\scriptsize{0.35} & 88.0$\pm$\scriptsize{0.11} & 86.3$\pm$\scriptsize{0.02 }& 93.1$\pm$\scriptsize{0.07} & 78.7$\pm$\scriptsize{0.59} & 90.5$\pm$\scriptsize{0.05} & 85.4 \\
                            & LoRA+WTA-CRS@0.3 & 60.0$\pm$\scriptsize{0.51} & 94.4$\pm$\scriptsize{0.16 }& 92.0$\pm$\scriptsize{0.38} & 87.3$\pm$\scriptsize{0.04} & 85.6$\pm$\scriptsize{0.08} & 93.2$\pm$\scriptsize{0.01} & 80.1$\pm$\scriptsize{1.02} & 90.4$\pm$\scriptsize{0.06} & 85.4 \\ \hline
\multirow{4}{*}{BERT-Large} & Full             & 66.8$\pm$\scriptsize{0.31} & 93.5$\pm$\scriptsize{0.29 }& 89.5$\pm$\scriptsize{0.26 }& 88.5$\pm$\scriptsize{0.03} & 86.4$\pm$\scriptsize{0.19} & 92.1$\pm$\scriptsize{0.24} & 72.6$\pm$\scriptsize{0.36} & 90.2$\pm$\scriptsize{0.76} & 85.0 \\
                            & LoRA             & 65.9$\pm$\scriptsize{0.27} & 93.8$\pm$\scriptsize{0.17 }& 90.8$\pm$\scriptsize{0.37} & 87.6$\pm$\scriptsize{0.08} & 85.9$\pm$\scriptsize{0.05} & 92.0$\pm$\scriptsize{0.2}  & 71.3$\pm$\scriptsize{0.18} & 90.3$\pm$\scriptsize{0.09} & 84.7 \\
                            & WTA-CRS@0.3      & 64.7$\pm$\scriptsize{0.44} & 93.5$\pm$\scriptsize{0.0  }& 89.3$\pm$\scriptsize{0.39} & 88.2$\pm$\scriptsize{0.04} & 85.2$\pm$\scriptsize{0.03} & 91.9$\pm$\scriptsize{0.12} & 73.8$\pm$\scriptsize{0.54} & 90.4$\pm$\scriptsize{0.02} & 84.6 \\
                            & LoRA+WTA-CRS@0.3 & 66.0$\pm$\scriptsize{0.33} & 93.3$\pm$\scriptsize{0.29 }& 89.7$\pm$\scriptsize{1.32} & 87.6$\pm$\scriptsize{0.02} & 86.0$\pm$\scriptsize{0.07} & 91.9$\pm$\scriptsize{0.14} & 72.4$\pm$\scriptsize{0.17} & 89.7$\pm$\scriptsize{0.04} & 84.6 \\ \hline
\multirow{5}{*}{T5-Large}   & Full             & 61.3$\pm$\scriptsize{1.01} & 96.3$\pm$\scriptsize{0.0 } & 93.4$\pm$\scriptsize{0.13} & 89.7$\pm$\scriptsize{0.01} & 89.8$\pm$\scriptsize{0.07} & 94.2$\pm$\scriptsize{0.05} & 85.3$\pm$\scriptsize{0.17} & 91.8$\pm$\scriptsize{0.08} & 87.7 \\
                            & LoRA             & 63.3$\pm$\scriptsize{0.26} & 96.4$\pm$\scriptsize{0.14} & 93.5$\pm$\scriptsize{0.16} & 88.5$\pm$\scriptsize{0.03} & 89.5$\pm$\scriptsize{0.05} & 94.3$\pm$\scriptsize{0.07} & 84.2$\pm$\scriptsize{0.68} & 91.7$\pm$\scriptsize{0.13} & 87.7 \\
                            & LST              & 59.9$\pm$\scriptsize{0.77} & 95.8$\pm$\scriptsize{0.06} & 91.8$\pm$\scriptsize{0.08} & 88.4$\pm$\scriptsize{0.01} & 88.7$\pm$\scriptsize{0.05} & 94.2$\pm$\scriptsize{0.02} & 82.5$\pm$\scriptsize{0.18} & 91.4$\pm$\scriptsize{0.07} & 86.6 \\
                            & WTA-CRS@0.3      & 60.9$\pm$\scriptsize{1.18} & 96.3$\pm$\scriptsize{0.25} & 93.6$\pm$\scriptsize{0.57 }& 89.3$\pm$\scriptsize{0.04} & 89.5$\pm$\scriptsize{0.12} & 94.1$\pm$\scriptsize{0.03} & 84.4$\pm$\scriptsize{0.34} & 91.3$\pm$\scriptsize{0.05} & 87.4 \\
                            & LoRA+WTA-CRS@0.3 & 62.9$\pm$\scriptsize{1.19} & 96.2$\pm$\scriptsize{0.05} & 93.6$\pm$\scriptsize{0.47} & 88.3$\pm$\scriptsize{0.02} & 89.2$\pm$\scriptsize{0.08} & 94.0$\pm$\scriptsize{0.07} & 83.9$\pm$\scriptsize{0.95} & 91.3$\pm$\scriptsize{0.03} & 87.4 \\ \hline
\multirow{2}{*}{T5-3B}      & LoRA             & 70.1$\pm$\scriptsize{0.37} & 96.8$\pm$\scriptsize{0.29} & 94.0$\pm$\scriptsize{0.27 }& 89.9$\pm$\scriptsize{0.0}  & 91.0$\pm$\scriptsize{0.14} & 95.6$\pm$\scriptsize{0.05} & 85.9$\pm$\scriptsize{0.36} & 92.9$\pm$\scriptsize{0.08} & 89.5 \\
                            & LoRA+WTA-CRS@0.3 & 71.4$\pm$\scriptsize{0.35} & 96.4$\pm$\scriptsize{0.06} & 94.6$\pm$\scriptsize{0.39} & 90.0$\pm$\scriptsize{0.05} & 91.0$\pm$\scriptsize{0.06} & 95.6$\pm$\scriptsize{0.12} & 86.3$\pm$\scriptsize{0.36} & 92.9$\pm$\scriptsize{0.09} & 89.8 \\ \hline
\end{tabular}
}
\end{table}

\begin{table}[h!]
    \centering
    \footnotesize
    \captionsetup{skip=5pt} 
    \caption{Peak memory usage (GB) and compression rate of fine-tuning T5-Base and -Large. 
    We measure the memory usage on a single NVIDIA A100 (80GB) GPU.
    For T5-3B, since it is trained using multi-GPUs with data parallel. 
    We instead report the maximum batch size in Figure \ref{fig:mem_vs_bs} for it.
    }
    \label{tab:compression}
\resizebox{\textwidth}{!}{
    \begin{tabular}{lccccccc}
    \toprule
         & FP & LoRA & LST & \Algnameabbr{}@0.3 & \Algnameabbr{}@0.1 & LoRA+\Algnameabbr{}@0.3 & LoRA+\Algnameabbr{}@0.1 \\
    \midrule
         T5-Base & 17.66~(1$\times$) & 13.84~(1.3$\times$) & 5.50~(3.2$\times$) & 8.44~(2.1$\times$) & 7.30~(2.4$\times$) & 6.50~(2.7$\times$) & 5.44~(3.2$\times$) \\
         T5-Large & 45.85~(1$\times$) & 36.83~(1.2$\times$) & 14.85~(3.1$\times$) & 21.58~(2.1$\times$) & 18.46~(2.5$\times$) & 17.44~(2.6$\times$) & 14.64~(3.13$\times$) \\
    \bottomrule
    \end{tabular}
    }
    \vspace{-1em}
\end{table}

\subsection{Accuracy versus Memory Usage (RQ1)}

To answer \textbf{RQ1}, 
we first analyze the trade-off between the model performance and memory saving.
The evaluation results and peak memory usage are given in Tables~\ref{tab:glue_eval} and~\ref{tab:compression}, respectively.
We observe:


\ding{182} \emph{\sas achieves a superior trade-off between accuracy and memory usage compared to baselines. Specifically, \sas has negligible accuracy drop, while the peak memory usage is reduced by $2.1\times\sim2.7\times$}  (when combined with LoRA).

\begin{wrapfigure}{r}{0.32\textwidth}
  \begin{minipage}[t]{0.95\linewidth}
    \includegraphics[width=1.0\textwidth]{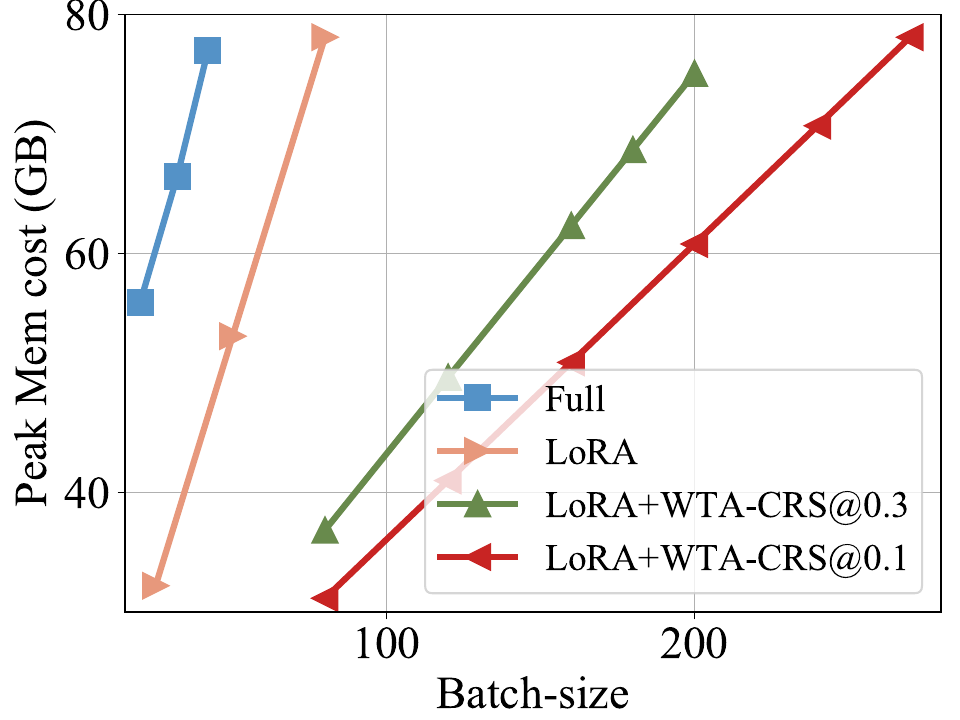}
    \caption{Peak memory usage versus the maximum batch size of T5-3B. More similar results are shown in Appendix \ref{app:more_exp_speed}.}
    \label{fig:mem_vs_bs}
  \end{minipage}
    \vspace{-1em}
\end{wrapfigure}

As we analyzed, LoRA mainly reduces the memory of optimizer states.
Thus, although it has negligible accuracy drop, it can only achieve $\sim1.3\times$ peak memory saving.
LST can reduce memory usage up to $3\times$, but its accuracy drop is much larger than LoRA and \sas.
Since \sas executes at the operation level and focuses on the activations,
we further combine LoRA with \sas to reduce memory usage more aggressively.
When combing with LoRA, \sas achieves $2.7\times$ memory usage saving with almost no accuracy drop.
To fully tune T5-3B, it requires 37.7GB of memory and relies on a GPU with a capacity of 40GB or higher, e.g. RTX8000, A100, or A40. On the other hand, LoRA+\sas only requires 21.6GB of memory for finetuning with a mini-batch size of $32$, which can run on a GPU with 24GB memory, e.g. RTX3090Ti or A5000. We have experimentally confirmed this conclusion.
Under the same hardware, \sas enables the tuning of larger models, resulting in improved down-streaming task performance. 
Thus as shown in Figure \ref{fig: accuracy-memory_trade-off}, \ding{183} \emph{under the similar memory budget, \sas outperforms the other methods in terms of the accuracy.}

Also, according to Figure \ref{fig:mem_vs_bs}, for T5-3B, LoRA itself can enable $1.9\times$ larger batch size. 
\ding{184} \emph{When combined with LoRA, \sas enables $4.8\times$ ($k\!=\!0.3|\mathcal{D}|$) to $6.4\times$ ($k\!=\!0.1|\mathcal{D}|$) larger batch-size.}

\paragraph{Influence of Row-column Pairs Budget (RQ2).}
As we analyzed in Section \ref{sec: where}, \sas only have one hyperparameter, i.e., the total column-row pair budgets $k$.  
We conduct the ablation study with different budget $k$ in Figure \ref{fig:hyper_exp}.
We observe that \ding{185} \emph{It has almost no accuracy drop when $k=0.3|\mathcal{D}|$.
And the accuracy drop is about $1\%$ when $k=0.1|\mathcal{D}|$}.
\emph{Notably, for T5-3B, the accuracy drop is only $0.4\%$ when  $k=0.1|\mathcal{D}|$, which is much smaller than T5-Base and T5-Large.}
This suggests that larger models are more compressible because they have more redundant activations, which is consistent with previous observations \citep{li2020train}.


    

    
    

\begin{figure}[h!]
    \centering
    \begin{subfigure}[h]{0.32\linewidth}
      \includegraphics[width=1\linewidth]{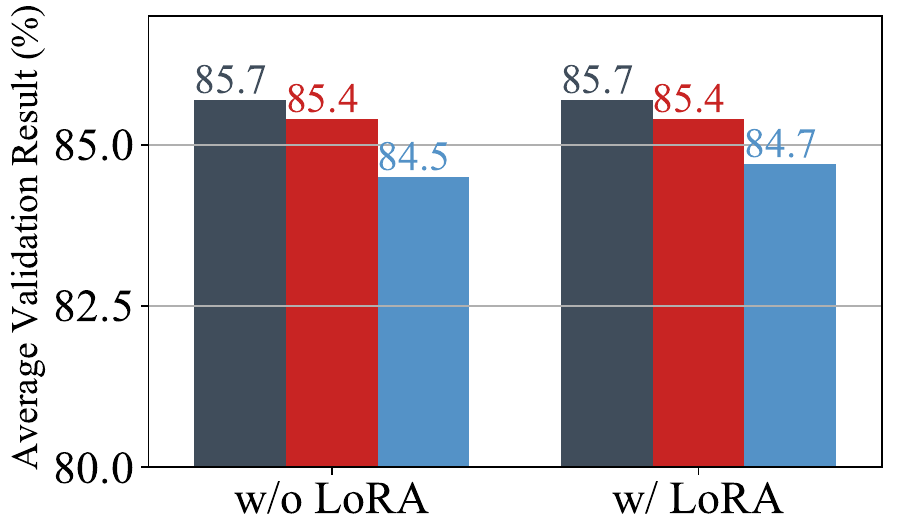}
      \caption{T5-Base}
    \end{subfigure}%
    \begin{subfigure}[h]{0.32\linewidth}
      \includegraphics[width=1\linewidth]{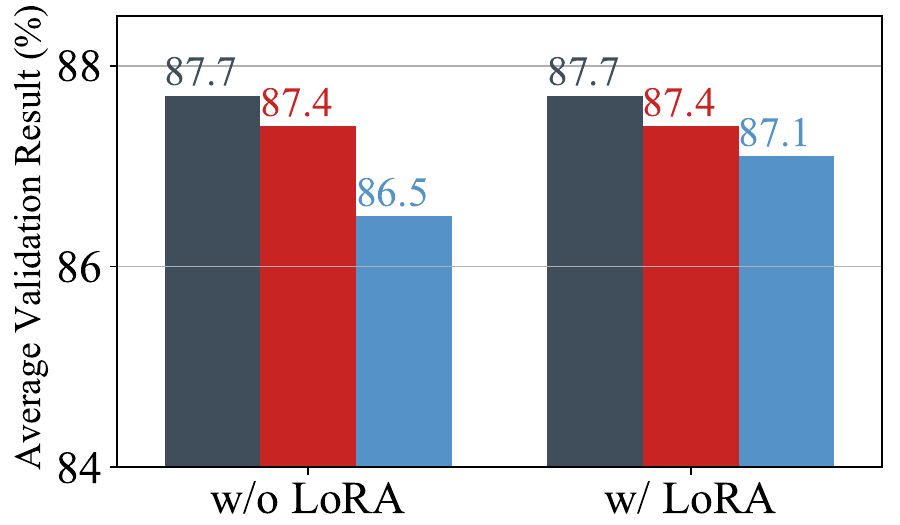}
      \caption{T5-Large}
    \end{subfigure}%
    \begin{subfigure}[h]{0.32\linewidth}
      \includegraphics[width=1\linewidth]{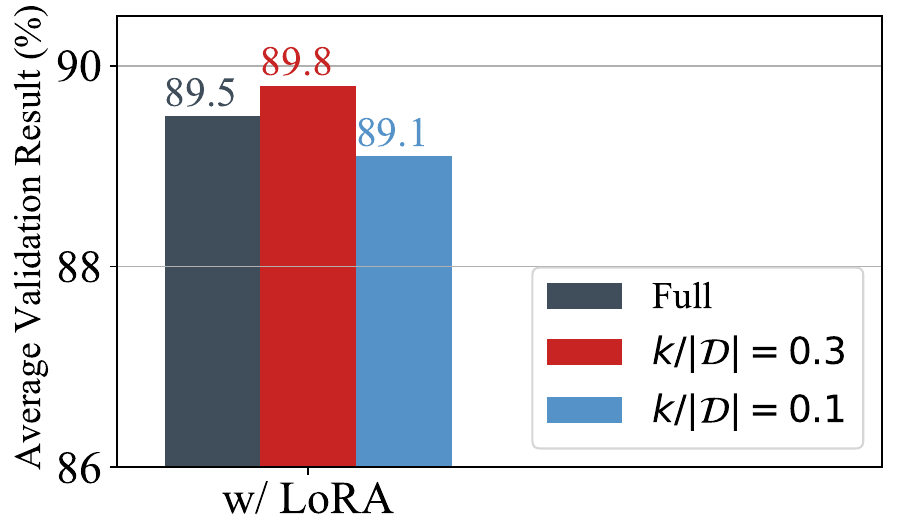}
      \caption{T5-3B}
    \end{subfigure}%
    \vspace{-.5em}
    \caption{Average validation results on GLUE dataset of \sas with varying budgets.}
    \vspace{-1em}
    \label{fig:hyper_exp}
\end{figure}

\subsection{Ablation Study (RQ3 and RQ4)}

\begin{figure}
    \centering
    \begin{subfigure}[h]{0.34\linewidth}
      \includegraphics[width=1\linewidth]{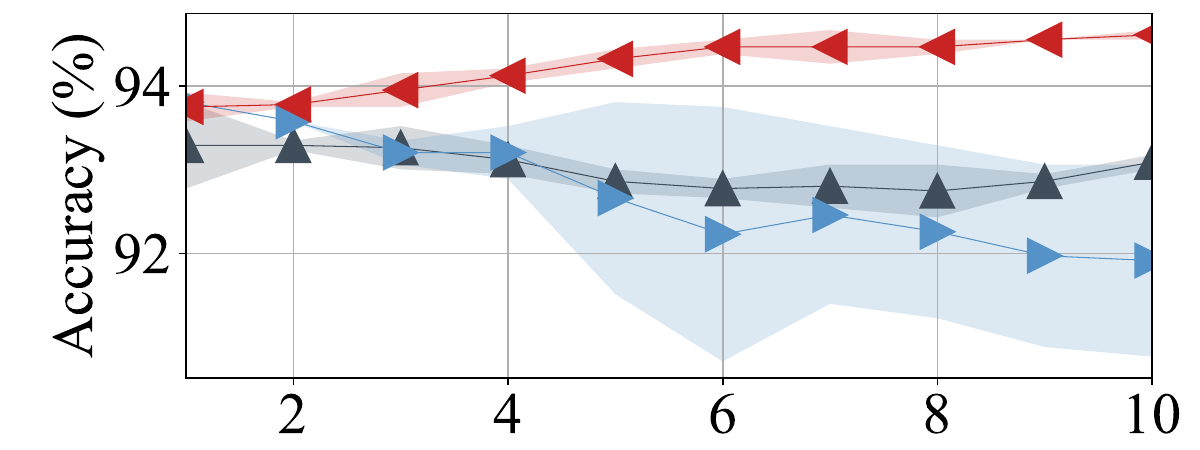}
          \vspace{-1.5em}
      \caption{SST2}
    \end{subfigure}%
    \begin{subfigure}[h]{0.34\linewidth}
      \includegraphics[width=1\linewidth]{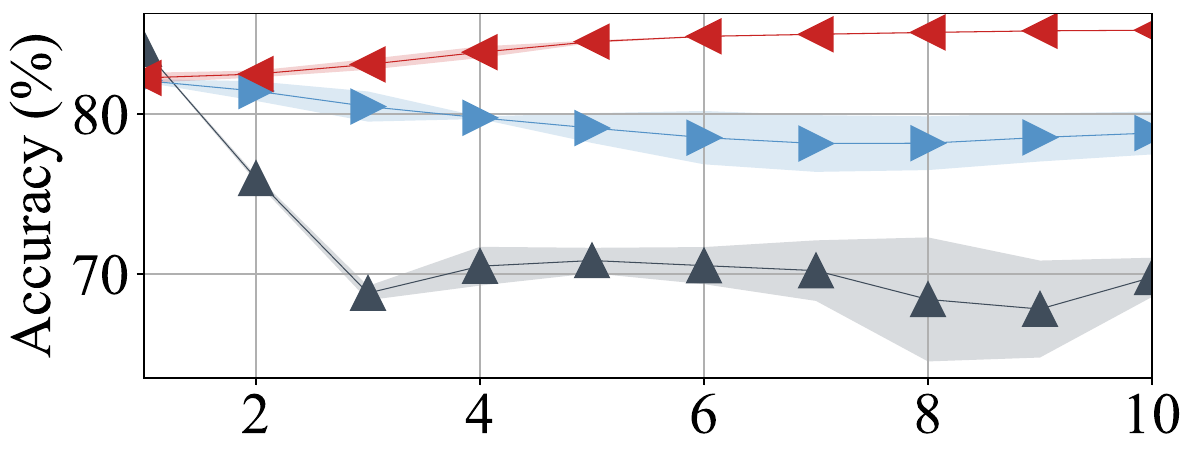}
          \vspace{-1.5em}
      \caption{MNLI}
    \end{subfigure}%
    \begin{subfigure}[h]{0.34\linewidth}
      \includegraphics[width=1\linewidth]{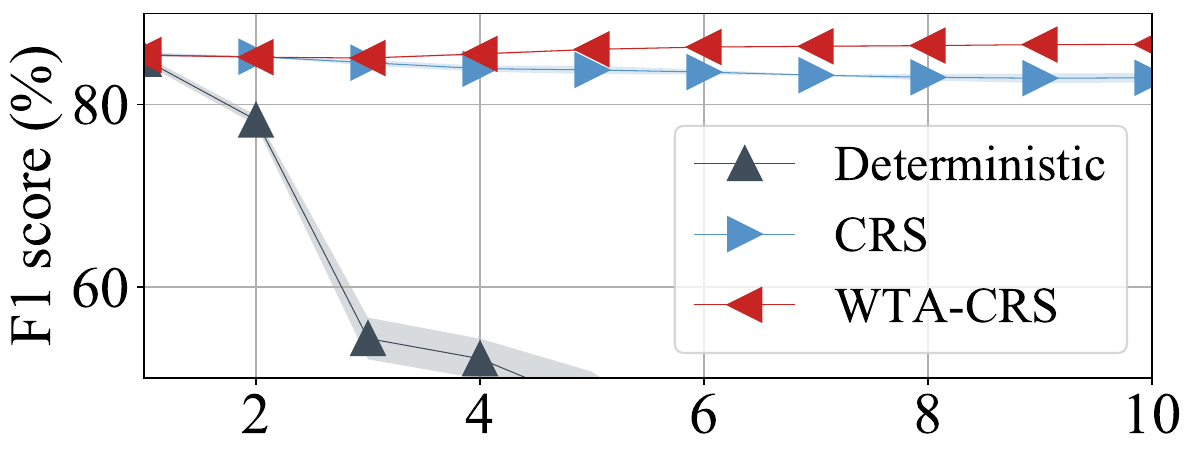}
          \vspace{-1.5em}
      \caption{QQP}
    \end{subfigure}%
    \vspace{-.5em}
    \caption{Validation results of T5-Base with different methods.}
    \vspace{-1.5em}
    \label{fig:ablation_exp}
\end{figure}

To answer \textbf{RQ3}, \sas is compared with two compositional methods to demonstrate its superiority. Namely, 
(1) the \textbf{Deterministic} method selects row-column pairs with top $k$ probability of Equation~(\ref{eq: col_row_norm}). 
We note that this is the estimator proposed in \citep{adelman2021faster}.
(2) The \textbf{CRS} method follows Equation~(\ref{eq: col_row_norm}) to sample the row-column pairs. 
All methods are deployed to \mm in the backward pass, while leaving the forward one unchanged.
The experiments are conducted on the training of T5-base language model on the SST2, MNLI, and QQP datasets; 
The column-row pair budget takes $k/|\mathcal{D}| = 0.1$ for all methods. 
The validation accuracy versus training epoch is given in Figure~\ref{fig:ablation_exp}.
We observe:

\ding{186} 
\emph{\sas outperforms all compared methods, especially as the training epoch grows.} 
The deterministic selection of top $k$ column-row pairs suffers from accumulation of bias error that ultimately results in a failure of convergence.
For CRS, it also enables the unbiased weight gradient. However, as we theoretically and experimentally analyzed in Theorem~\ref{eq: var_thresh} and Figure~\ref{fig: exp_assump_analysis_th2_direct}, it is worse than \sas due to larger variance.
In summary, both the deterministic and  stochastic parts contribute to the effectiveness of \sas, which is consistent with our theoretical analysis.


\begin{wrapfigure}{r}{0.52\textwidth}
\vspace{-1em}
  \begin{subfigure}[h]{0.51\linewidth}
    \centering
    \includegraphics[width=\linewidth]{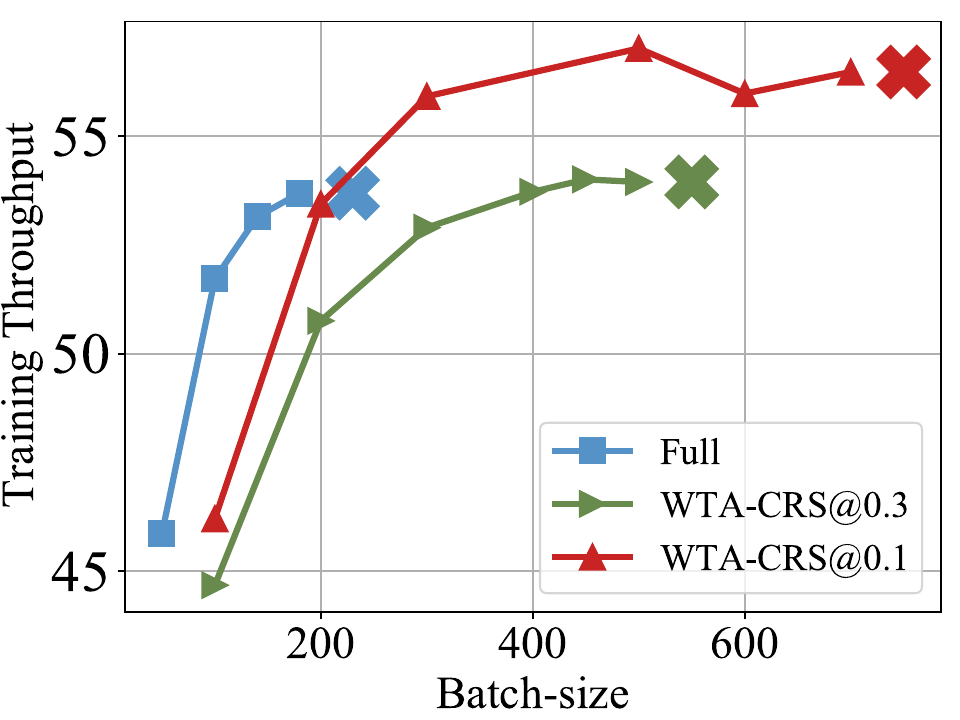}
            \vspace{-1.5em}
    \caption{T5-Large}
    \label{fig:mot_quant}
  \end{subfigure}%
  \begin{subfigure}[h]{0.51\linewidth}
    \centering
    \includegraphics[width=\linewidth]{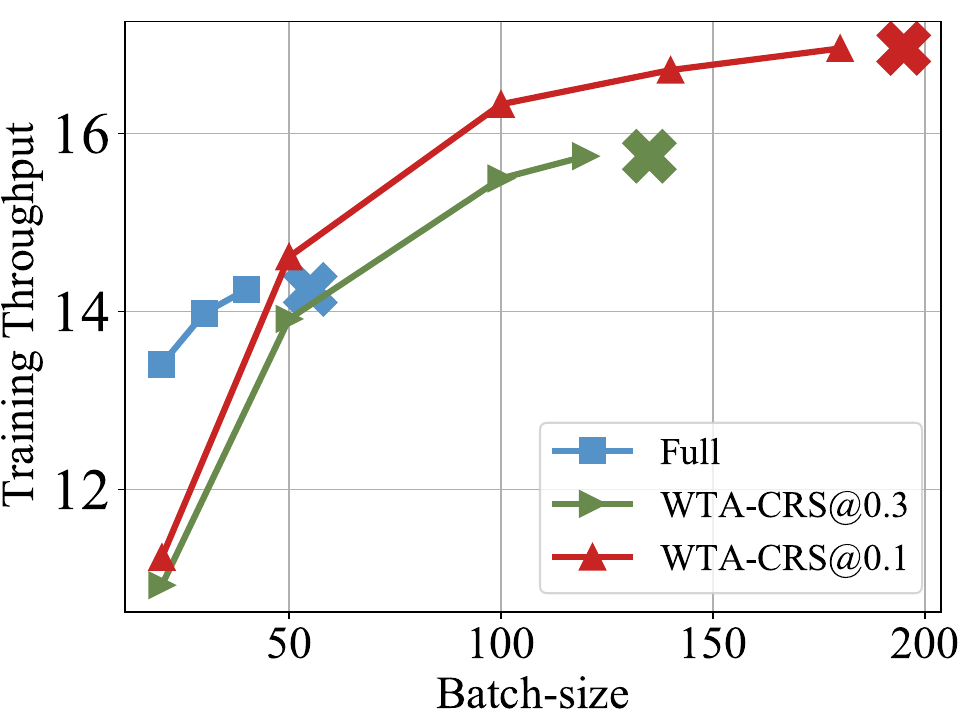}
            \vspace{-1.5em}
        \caption{T5-3B}
    \label{fig:mot_sparse}
  \end{subfigure}
        \vspace{-.5em}
  \caption{Batch size versus training throughput (sentences/sec) with different methods, where the sequential length is 128. The hardware is one single NVIDIA-A100 (80GB).}
    \vspace{-1.5em}
\label{fig:throughput}
\end{wrapfigure}

\paragraph{The speed of \sas (RQ4).}

The configuration of computational infrastructure is given in Appendix~\ref{appendix:hardware}.
We note that \sas does not add any extra parameters to the model.
Thus, \sas only affect the fine-tuning speed, without affecting the inference speed.
The memory taken by the activations is proportional to the batch size and sequential length. 
Below we analyze how the fine-tuning speed affected by \sas.
As we analyzed in Appendix \ref{app: related_work} limitation, the current implementation is not heavily optimized and thus the execution time of \sas is still slower than the original linear operation (details are shown in Appendix \ref{app:more_exp_speed}).
However, under the same hardware, a reduction in activation memory enables the use of larger batch sizes, thereby improving training speed due to increased GPU utilization \citep{largebatchsizetraining, lst}.
As we analyzed in Figure \ref{fig:mem_vs_bs}, \sas can enlarge the available batch size by up to $4.8\times$ larger.
This enhancement is expected to result in a acceleration of the training speed. 
To illustrate this relationship, Figure \ref{fig:throughput} presents a visualization of batch size against training throughput (sentences per second) for both T5-Large and T5-3B models.
We observe that \ding{187}
\emph{\sas enables faster training speed under the same hardware.}
Specifically, on the T5-Large model, \Algnameabbr{}@0.1 shows $1.08\times$ higher training throughput; and on the T5-3B model, \Algnameabbr{}@0.3 and \Algnameabbr{}@0.1 achieve $1.14\times$ and $1.21\times$ higher training throughput, respectively.

\section{Conclusion}
In this paper, we propose \sas, a new unbiased estimator for matrix production with reduced variance.
We theoretically and experimentally show when and why the estimator is better than the traditional unbiased estimator in terms of the variance.
In the context of adapting transformers, it almost has no accuracy drop while reducing the peak memory usage by up to $2.7\times$, and it enables a $6.4\times$ larger batch size, which in return resulting in $1.2\times$ higher training throughput.

\bibliographystyle{plainnat}
\bibliography{ref}

\begin{thebibliography}{47}
\providecommand{\natexlab}[1]{#1}
\providecommand{\url}[1]{\texttt{#1}}
\expandafter\ifx\csname urlstyle\endcsname\relax
  \providecommand{\doi}[1]{doi: #1}\else
  \providecommand{\doi}{doi: \begingroup \urlstyle{rm}\Url}\fi

\bibitem[Abadi et~al.(2016)Abadi, Barham, Chen, Chen, Davis, Dean, Devin,
  Ghemawat, Irving, Isard, et~al.]{abadi2016tensorflow}
Mart{\'\i}n Abadi, Paul Barham, Jianmin Chen, Zhifeng Chen, Andy Davis, Jeffrey
  Dean, Matthieu Devin, Sanjay Ghemawat, Geoffrey Irving, Michael Isard, et~al.
\newblock Tensorflow: a system for large-scale machine learning.
\newblock In \emph{Osdi}, volume~16, pages 265--283. Savannah, GA, USA, 2016.

\bibitem[Adelman et~al.(2021)Adelman, Levy, Hakimi, and
  Silberstein]{adelman2021faster}
Menachem Adelman, Kfir Levy, Ido Hakimi, and Mark Silberstein.
\newblock Faster neural network training with approximate tensor operations.
\newblock \emph{Advances in Neural Information Processing Systems},
  34:\penalty0 27877--27889, 2021.

\bibitem[Andoorveedu et~al.(2022{\natexlab{a}})Andoorveedu, Zhu, Zheng, and
  Pekhimenko]{andoorveedu2022tempo}
Muralidhar Andoorveedu, Zhanda Zhu, Bojian Zheng, and Gennady Pekhimenko.
\newblock Tempo: Accelerating transformer-based model training through memory
  footprint reduction.
\newblock \emph{arXiv preprint arXiv:2210.10246}, 2022{\natexlab{a}}.

\bibitem[Andoorveedu et~al.(2022{\natexlab{b}})Andoorveedu, Zhu, Zheng, and
  Pekhimenko]{tempo}
Muralidhar Andoorveedu, Zhanda Zhu, Bojian Zheng, and Gennady Pekhimenko.
\newblock Tempo: Accelerating transformer-based model training through memory
  footprint reduction.
\newblock \emph{arXiv preprint arXiv:2210.10246}, 2022{\natexlab{b}}.

\bibitem[Brown et~al.(2020)Brown, Mann, Ryder, Subbiah, Kaplan, Dhariwal,
  Neelakantan, Shyam, Sastry, Askell, et~al.]{gpt3}
Tom Brown, Benjamin Mann, Nick Ryder, Melanie Subbiah, Jared~D Kaplan, Prafulla
  Dhariwal, Arvind Neelakantan, Pranav Shyam, Girish Sastry, Amanda Askell,
  et~al.
\newblock Language models are few-shot learners.
\newblock \emph{Advances in neural information processing systems},
  33:\penalty0 1877--1901, 2020.

\bibitem[Chen et~al.(2021{\natexlab{a}})Chen, Dao, Liang, Yang, Song, Rudra,
  and Re]{chen2021pixelated}
Beidi Chen, Tri Dao, Kaizhao Liang, Jiaming Yang, Zhao Song, Atri Rudra, and
  Christopher Re.
\newblock Pixelated butterfly: Simple and efficient sparse training for neural
  network models.
\newblock \emph{arXiv preprint arXiv:2112.00029}, 2021{\natexlab{a}}.

\bibitem[Chen et~al.(2021{\natexlab{b}})Chen, Liu, Peng, Xu, Li, Dao, Song,
  Shrivastava, and Re]{clp+21}
Beidi Chen, Zichang Liu, Binghui Peng, Zhaozhuo Xu, Jonathan~Lingjie Li, Tri
  Dao, Zhao Song, Anshumali Shrivastava, and Christopher Re.
\newblock {MONGOOSE}: A learnable {LSH} framework for efficient neural network
  training.
\newblock In \emph{International Conference on Learning Representations
  (ICLR)}, 2021{\natexlab{b}}.
\newblock URL \url{https://openreview.net/forum?id=wWK7yXkULyh}.

\bibitem[Chen et~al.(2021{\natexlab{c}})Chen, Zheng, Yao, Wang, Stoica,
  Mahoney, and Gonzalez]{actnn}
Jianfei Chen, Lianmin Zheng, Zhewei Yao, Dequan Wang, Ion Stoica, Michael
  Mahoney, and Joseph Gonzalez.
\newblock Actnn: Reducing training memory footprint via 2-bit activation
  compressed training.
\newblock In \emph{International Conference on Machine Learning}, pages
  1803--1813. PMLR, 2021{\natexlab{c}}.

\bibitem[Chmiel et~al.(2021)Chmiel, Banner, Hoffer, Yaacov, and
  Soudry]{chmiel2021logarithmic}
Brian Chmiel, Ron Banner, Elad Hoffer, Hilla~Ben Yaacov, and Daniel Soudry.
\newblock Logarithmic unbiased quantization: Simple 4-bit training in deep
  learning.
\newblock \emph{arXiv preprint arXiv:2112.10769}, 2021.

\bibitem[Chmiel et~al.(2023)Chmiel, Hubara, Banner, and Soudry]{chmielminimum}
Brian Chmiel, Itay Hubara, Ron Banner, and Daniel Soudry.
\newblock Minimum variance unbiased n: M sparsity for the neural gradients.
\newblock In \emph{The Eleventh International Conference on Learning
  Representations}, 2023.

\bibitem[Chuang et~al.(2023)Chuang, Wang, Yang, Liu, Cai, Du, and
  Hu]{DBLP:journals/corr/abs-2302-03225}
Yu{-}Neng Chuang, Guanchu Wang, Fan Yang, Zirui Liu, Xuanting Cai, Mengnan Du,
  and Xia Hu.
\newblock Efficient {XAI} techniques: {A} taxonomic survey.
\newblock \emph{CoRR}, abs/2302.03225, 2023.
\newblock \doi{10.48550/arXiv.2302.03225}.
\newblock URL \url{https://doi.org/10.48550/arXiv.2302.03225}.

\bibitem[Dao et~al.(2022)Dao, Chen, Sohoni, Desai, Poli, Grogan, Liu, Rao,
  Rudra, and R{\'e}]{dao2022monarch}
Tri Dao, Beidi Chen, Nimit~S Sohoni, Arjun Desai, Michael Poli, Jessica Grogan,
  Alexander Liu, Aniruddh Rao, Atri Rudra, and Christopher R{\'e}.
\newblock Monarch: Expressive structured matrices for efficient and accurate
  training.
\newblock In \emph{International Conference on Machine Learning}, pages
  4690--4721. PMLR, 2022.

\bibitem[Devlin et~al.(2018)Devlin, Chang, Lee, and Toutanova]{devlin2018bert}
Jacob Devlin, Ming-Wei Chang, Kenton Lee, and Kristina Toutanova.
\newblock Bert: Pre-training of deep bidirectional transformers for language
  understanding.
\newblock \emph{arXiv preprint arXiv:1810.04805}, 2018.

\bibitem[Drineas and Kannan(2001)]{drineas2001fast}
Petros Drineas and Ravi Kannan.
\newblock Fast monte-carlo algorithms for approximate matrix multiplication.
\newblock In \emph{Proceedings 42nd IEEE Symposium on Foundations of Computer
  Science}, pages 452--459. IEEE, 2001.

\bibitem[Drineas et~al.(2006)Drineas, Kannan, and Mahoney]{mcmatmul}
Petros Drineas, Ravi Kannan, and Michael~W Mahoney.
\newblock Fast monte carlo algorithms for matrices i: Approximating matrix
  multiplication.
\newblock \emph{SIAM Journal on Computing}, 36\penalty0 (1):\penalty0 132--157,
  2006.

\bibitem[Goyal et~al.(2017)Goyal, Doll{\'a}r, Girshick, Noordhuis, Wesolowski,
  Kyrola, Tulloch, Jia, and He]{largebatchsizetraining}
Priya Goyal, Piotr Doll{\'a}r, Ross Girshick, Pieter Noordhuis, Lukasz
  Wesolowski, Aapo Kyrola, Andrew Tulloch, Yangqing Jia, and Kaiming He.
\newblock Accurate, large minibatch sgd: Training imagenet in 1 hour.
\newblock \emph{arXiv preprint arXiv:1706.02677}, 2017.

\bibitem[Houlsby et~al.(2019)Houlsby, Giurgiu, Jastrzebski, Morrone,
  De~Laroussilhe, Gesmundo, Attariyan, and Gelly]{adapter}
Neil Houlsby, Andrei Giurgiu, Stanislaw Jastrzebski, Bruna Morrone, Quentin
  De~Laroussilhe, Andrea Gesmundo, Mona Attariyan, and Sylvain Gelly.
\newblock Parameter-efficient transfer learning for nlp.
\newblock In \emph{International Conference on Machine Learning}, pages
  2790--2799. PMLR, 2019.

\bibitem[Hsu et~al.(2024)Hsu, Dai, Kothapalli, Song, Tang, Zhu, Shimizu, Sahni,
  Ning, and Chen]{hsu2024ligerkernelefficienttriton}
Pin-Lun Hsu, Yun Dai, Vignesh Kothapalli, Qingquan Song, Shao Tang, Siyu Zhu,
  Steven Shimizu, Shivam Sahni, Haowen Ning, and Yanning Chen.
\newblock Liger kernel: Efficient triton kernels for llm training.
\newblock \emph{arXiv preprint arXiv:2410.10989}, 2024.
\newblock URL \url{https://arxiv.org/abs/2410.10989}.

\bibitem[Hu et~al.(2021)Hu, Shen, Wallis, Allen-Zhu, Li, Wang, Wang, and
  Chen]{hu2021lora}
Edward~J Hu, Yelong Shen, Phillip Wallis, Zeyuan Allen-Zhu, Yuanzhi Li, Shean
  Wang, Lu~Wang, and Weizhu Chen.
\newblock Lora: Low-rank adaptation of large language models.
\newblock \emph{arXiv preprint arXiv:2106.09685}, 2021.

\bibitem[Jain et~al.(2020)Jain, Jain, Nrusimha, Gholami, Abbeel, Gonzalez,
  Keutzer, and Stoica]{checkmate}
Paras Jain, Ajay Jain, Aniruddha Nrusimha, Amir Gholami, Pieter Abbeel, Joseph
  Gonzalez, Kurt Keutzer, and Ion Stoica.
\newblock Checkmate: Breaking the memory wall with optimal tensor
  rematerialization.
\newblock \emph{Proceedings of Machine Learning and Systems}, 2:\penalty0
  497--511, 2020.

\bibitem[Kaplan et~al.(2020)Kaplan, McCandlish, Henighan, Brown, Chess, Child,
  Gray, Radford, Wu, and Amodei]{kaplan2020scaling}
Jared Kaplan, Sam McCandlish, Tom Henighan, Tom~B Brown, Benjamin Chess, Rewon
  Child, Scott Gray, Alec Radford, Jeffrey Wu, and Dario Amodei.
\newblock Scaling laws for neural language models.
\newblock \emph{arXiv preprint arXiv:2001.08361}, 2020.

\bibitem[Karimi~Mahabadi et~al.(2021)Karimi~Mahabadi, Henderson, and
  Ruder]{karimi2021compacter}
Rabeeh Karimi~Mahabadi, James Henderson, and Sebastian Ruder.
\newblock Compacter: Efficient low-rank hypercomplex adapter layers.
\newblock \emph{Advances in Neural Information Processing Systems},
  34:\penalty0 1022--1035, 2021.

\bibitem[Kingma and Ba(2014)]{kingma2014adam}
Diederik~P Kingma and Jimmy Ba.
\newblock Adam: A method for stochastic optimization.
\newblock \emph{arXiv preprint arXiv:1412.6980}, 2014.

\bibitem[Kirisame et~al.(2020)Kirisame, Lyubomirsky, Haan, Brennan, He, Roesch,
  Chen, and Tatlock]{dtr}
Marisa Kirisame, Steven Lyubomirsky, Altan Haan, Jennifer Brennan, Mike He,
  Jared Roesch, Tianqi Chen, and Zachary Tatlock.
\newblock Dynamic tensor rematerialization.
\newblock \emph{arXiv preprint arXiv:2006.09616}, 2020.

\bibitem[Lester et~al.(2021)Lester, Al-Rfou, and Constant]{softprompt}
Brian Lester, Rami Al-Rfou, and Noah Constant.
\newblock The power of scale for parameter-efficient prompt tuning.
\newblock \emph{arXiv preprint arXiv:2104.08691}, 2021.

\bibitem[Li and Liang(2021)]{prefix}
Xiang~Lisa Li and Percy Liang.
\newblock Prefix-tuning: Optimizing continuous prompts for generation.
\newblock \emph{arXiv preprint arXiv:2101.00190}, 2021.

\bibitem[Li et~al.(2020)Li, Wallace, Shen, Lin, Keutzer, Klein, and
  Gonzalez]{li2020train}
Zhuohan Li, Eric Wallace, Sheng Shen, Kevin Lin, Kurt Keutzer, Dan Klein, and
  Joey Gonzalez.
\newblock Train big, then compress: Rethinking model size for efficient
  training and inference of transformers.
\newblock In \emph{International Conference on machine learning}, pages
  5958--5968. PMLR, 2020.

\bibitem[Liu et~al.(2022{\natexlab{a}})Liu, Zheng, Wang, Cen, Chen, Han, Chen,
  Liu, Tang, Gonzalez, et~al.]{gact}
Xiaoxuan Liu, Lianmin Zheng, Dequan Wang, Yukuo Cen, Weize Chen, Xu~Han,
  Jianfei Chen, Zhiyuan Liu, Jie Tang, Joey Gonzalez, et~al.
\newblock Gact: Activation compressed training for generic network
  architectures.
\newblock In \emph{International Conference on Machine Learning}, pages
  14139--14152. PMLR, 2022{\natexlab{a}}.

\bibitem[Liu et~al.(2020)Liu, Song, Zhou, Wang, Shan, and
  Hu]{DBLP:journals/corr/abs-2010-13015}
Zirui Liu, Qingquan Song, Kaixiong Zhou, Ting{-}Hsiang Wang, Ying Shan, and Xia
  Hu.
\newblock Towards interaction detection using topological analysis on neural
  networks.
\newblock \emph{CoRR}, abs/2010.13015, 2020.
\newblock URL \url{https://arxiv.org/abs/2010.13015}.

\bibitem[Liu et~al.(2021)Liu, Jin, Wang, Zhou, and Hu]{liu2021divaug}
Zirui Liu, Haifeng Jin, Ting-Hsiang Wang, Kaixiong Zhou, and Xia Hu.
\newblock Divaug: plug-in automated data augmentation with explicit diversity
  maximization.
\newblock In \emph{Proceedings of the IEEE/CVF International Conference on
  Computer Vision}, pages 4762--4770, 2021.

\bibitem[Liu et~al.(2022{\natexlab{b}})Liu, Chen, Zhou, Zha, Huang, and
  Hu]{liu2022rsc}
Zirui Liu, Shengyuan Chen, Kaixiong Zhou, Daochen Zha, Xiao Huang, and Xia Hu.
\newblock Rsc: Accelerating graph neural networks training via randomized
  sparse computations.
\newblock \emph{arXiv preprint arXiv:2210.10737}, 2022{\natexlab{b}}.

\bibitem[Liu et~al.(2022{\natexlab{c}})Liu, Zhou, Yang, Li, Chen, and
  Hu]{exact}
Zirui Liu, Kaixiong Zhou, Fan Yang, Li~Li, Rui Chen, and Xia Hu.
\newblock {EXACT}: Scalable graph neural networks training via extreme
  activation compression.
\newblock In \emph{International Conference on Learning Representations},
  2022{\natexlab{c}}.
\newblock URL \url{https://openreview.net/forum?id=vkaMaq95_rX}.

\bibitem[Oktay et~al.(2020)Oktay, McGreivy, Aduol, Beatson, and
  Adams]{oktay2020randomized}
Deniz Oktay, Nick McGreivy, Joshua Aduol, Alex Beatson, and Ryan~P Adams.
\newblock Randomized automatic differentiation.
\newblock \emph{arXiv preprint arXiv:2007.10412}, 2020.

\bibitem[Paszke et~al.(2019)Paszke, Gross, Massa, Lerer, Bradbury, Chanan,
  Killeen, Lin, Gimelshein, Antiga, et~al.]{paszke2019pytorch}
Adam Paszke, Sam Gross, Francisco Massa, Adam Lerer, James Bradbury, Gregory
  Chanan, Trevor Killeen, Zeming Lin, Natalia Gimelshein, Luca Antiga, et~al.
\newblock Pytorch: An imperative style, high-performance deep learning library.
\newblock \emph{Advances in neural information processing systems}, 32, 2019.

\bibitem[Raffel et~al.(2020{\natexlab{a}})Raffel, Shazeer, Roberts, Lee,
  Narang, Matena, Zhou, Li, and Liu]{raffel2020exploring}
Colin Raffel, Noam Shazeer, Adam Roberts, Katherine Lee, Sharan Narang, Michael
  Matena, Yanqi Zhou, Wei Li, and Peter~J Liu.
\newblock Exploring the limits of transfer learning with a unified text-to-text
  transformer.
\newblock \emph{The Journal of Machine Learning Research}, 21\penalty0
  (1):\penalty0 5485--5551, 2020{\natexlab{a}}.

\bibitem[Raffel et~al.(2020{\natexlab{b}})Raffel, Shazeer, Roberts, Lee,
  Narang, Matena, Zhou, Li, and Liu]{t5}
Colin Raffel, Noam Shazeer, Adam Roberts, Katherine Lee, Sharan Narang, Michael
  Matena, Yanqi Zhou, Wei Li, and Peter~J Liu.
\newblock Exploring the limits of transfer learning with a unified text-to-text
  transformer.
\newblock \emph{The Journal of Machine Learning Research}, 21\penalty0
  (1):\penalty0 5485--5551, 2020{\natexlab{b}}.

\bibitem[Sung et~al.(2022)Sung, Cho, and Bansal]{lst}
Yi-Lin Sung, Jaemin Cho, and Mohit Bansal.
\newblock Lst: Ladder side-tuning for parameter and memory efficient transfer
  learning.
\newblock \emph{arXiv preprint arXiv:2206.06522}, 2022.

\bibitem[Vaswani et~al.(2017)Vaswani, Shazeer, Parmar, Uszkoreit, Jones, Gomez,
  Kaiser, and Polosukhin]{vaswani2017attention}
Ashish Vaswani, Noam Shazeer, Niki Parmar, Jakob Uszkoreit, Llion Jones,
  Aidan~N Gomez, {\L}ukasz Kaiser, and Illia Polosukhin.
\newblock Attention is all you need.
\newblock \emph{Advances in neural information processing systems}, 30, 2017.

\bibitem[Wang et~al.(2018{\natexlab{a}})Wang, Singh, Michael, Hill, Levy, and
  Bowman]{glue}
Alex Wang, Amanpreet Singh, Julian Michael, Felix Hill, Omer Levy, and Samuel~R
  Bowman.
\newblock Glue: A multi-task benchmark and analysis platform for natural
  language understanding.
\newblock \emph{arXiv preprint arXiv:1804.07461}, 2018{\natexlab{a}}.

\bibitem[Wang et~al.(2018{\natexlab{b}})Wang, Singh, Michael, Hill, Levy, and
  Bowman]{wang2018glue}
Alex Wang, Amanpreet Singh, Julian Michael, Felix Hill, Omer Levy, and Samuel~R
  Bowman.
\newblock Glue: A multi-task benchmark and analysis platform for natural
  language understanding.
\newblock \emph{arXiv preprint arXiv:1804.07461}, 2018{\natexlab{b}}.

\bibitem[Wang et~al.(2022{\natexlab{a}})Wang, Bhat, Jiang, Chen, Zha, Reyes,
  Niktash, Ulkar, Okman, Cai, et~al.]{wang2022bed}
Guanchu Wang, Zaid~Pervaiz Bhat, Zhimeng Jiang, Yi-Wei Chen, Daochen Zha,
  Alfredo~Costilla Reyes, Afshin Niktash, Gorkem Ulkar, Erman Okman, Xuanting
  Cai, et~al.
\newblock Bed: A real-time object detection system for edge devices.
\newblock In \emph{Proceedings of the 31st ACM International Conference on
  Information \& Knowledge Management}, pages 4994--4998, 2022{\natexlab{a}}.

\bibitem[Wang et~al.(2022{\natexlab{b}})Wang, Liu, Jiang, Liu, Zou, and
  Hu]{wang2022towards}
Guanchu Wang, Zirui Liu, Zhimeng Jiang, Ninghao Liu, Na~Zou, and Xia Hu.
\newblock Towards memory efficient training via dual activation precision.
\newblock \emph{arXiv preprint arXiv:2208.04187}, 2022{\natexlab{b}}.

\bibitem[Wolf et~al.(2020)Wolf, Debut, Sanh, Chaumond, Delangue, Moi, Cistac,
  Rault, Louf, Funtowicz, et~al.]{huggingface}
Thomas Wolf, Lysandre Debut, Victor Sanh, Julien Chaumond, Clement Delangue,
  Anthony Moi, Pierric Cistac, Tim Rault, R{\'e}mi Louf, Morgan Funtowicz,
  et~al.
\newblock Transformers: State-of-the-art natural language processing.
\newblock In \emph{Proceedings of the 2020 conference on empirical methods in
  natural language processing: system demonstrations}, pages 38--45, 2020.

\bibitem[Yang et~al.(2023)Yang, Jin, Tang, Han, Feng, Jiang, Yin, and
  Hu]{yang2023harnessing}
Jingfeng Yang, Hongye Jin, Ruixiang Tang, Xiaotian Han, Qizhang Feng, Haoming
  Jiang, Bing Yin, and Xia Hu.
\newblock Harnessing the power of llms in practice: A survey on chatgpt and
  beyond.
\newblock \emph{arXiv preprint arXiv:2304.13712}, 2023.

\bibitem[Zaken et~al.(2021)Zaken, Ravfogel, and Goldberg]{bitfit}
Elad~Ben Zaken, Shauli Ravfogel, and Yoav Goldberg.
\newblock Bitfit: Simple parameter-efficient fine-tuning for transformer-based
  masked language-models.
\newblock \emph{arXiv preprint arXiv:2106.10199}, 2021.

\bibitem[Zha et~al.(2023)Zha, Bhat, Lai, Yang, Jiang, Zhong, and
  Hu]{zha2023data}
Daochen Zha, Zaid~Pervaiz Bhat, Kwei-Herng Lai, Fan Yang, Zhimeng Jiang,
  Shaochen Zhong, and Xia Hu.
\newblock Data-centric artificial intelligence: A survey.
\newblock \emph{arXiv preprint arXiv:2303.10158}, 2023.

\bibitem[Zhong et~al.(2022)Zhong, Zhang, Huang, and Xu]{zhong2022revisit}
Shaochen Zhong, Guanqun Zhang, Ningjia Huang, and Shuai Xu.
\newblock Revisit kernel pruning with lottery regulated grouped convolutions.
\newblock In \emph{International Conference on Learning Representations}, 2022.
\newblock URL \url{https://openreview.net/forum?id=LdEhiMG9WLO}.

\end{thebibliography}

\clearpage
\appendix
\section*{Appendix}

\section{Extended Related Work and Discussion}
\label{app: related_work}

\paragraph{Parameter-Efficient Fine-tuning.}
Parameter-efficient tuning methods select a small subset of parameters or insert a few parameters to a pre-trained network. 
Then they only update the small subset of parameters,
while keeping others fixed \citep{softprompt, lst, prefix, bitfit, hu2021lora, karimi2021compacter, adapter,clp+21}. 
For example,
Adapters \citep{adapter, karimi2021compacter} insert a small module into the transformer blocks and only update it. 
Similarly, prompt tuning \citep{softprompt} introduces a small vector that is concatenated with the input embeddings.
BitFit \citep{bitfit} only tunes the bias term of the model.
LoRA \citep{hu2021lora}  injects trainable rank decomposition matrices into the transformer block.
Although these methods are ``parameter-efficient'', they actually cannot reduce the memory usage of the model itself. 
This is because we still needs to build the computation graph for the whole model.
Instead, the memory usage of optimizer states will be significantly reduced, which is in proportional to the number of trainable parameters \citep{kingma2014adam}.

\paragraph{Gradient Checkpointing.}

Gradient checkpointing helps decrease activation memory usage by saving only a selection of activations. However, it demands additional computation during the backward pass, as discarded activations must be recalculated \citep{dtr, checkmate}.
According to the report of Checkmate\footnote{\url{https://github.com/parasj/checkmate/issues/153}} \citep{checkmate}, 
it achieves ``a 2.3x memory reduction when training a BERT model with Checkmate optimizations (at 1x extra overhead for rematerialization)''.

\paragraph{Limitations}
Although \sas significantly reduces the computation of the backward pass in a hardware-friendly way i.e., dropping entire rows/columns in the tensor, the current implementation still hampers the execution time of linear operations. 
This is because the extra sampling process and data movement counteract the acceleration.
However, we note that (1) the overhead can be greatly reduced with better implementation, e.g., using prefetch and operation-fusion technique \citep{gact}; (2)  the existing implementation can still yield a large speedup when employing larger batch sizes (Figure \ref{fig:throughput}).

\paragraph{Potential Negative Societal Impacts.}
Our research primarily focuses on reducing the memory requirement of fine-tuning Language Models (LMs). 
The carbon emissions produced by LM fine-tuning may pose environmental issues. Our next step is to further improve the efficiency of LM fine-tuning, particularly on hardware with lower energy consumption.

\section{Unbiasedness of Weight Gradient}
\label{app: unbias_grad}

This part we directly follow the proof of Theorem 1 in ActNN \citep{actnn}.
For completeness, we provide the proof sketch here that is short and easy to follow.
Specifically, here we use \relu as the activation function for illustration convenience. 
We note that the conclusion in this section holds for any non-linear activation function.
Specifically, the forward pass of ReLU-Linear at the $l^{\text{th}}$ layer is 
\begin{align}
\mZ^{(l+1)} &=\mH^{(l)}\mW^{(l)}, \nonumber\\
\mH^{(l+1)} &=\relu(\mZ^{(l+1)}), \nonumber
\end{align}

and the backward pass of \relu is:

\begin{align*}
\E[\nabla \mZ^{(l+1)}] 
&=\E[\mathbbm{1}_{\mZ^{(l+1)}>0}\odot\nabla \mH^{(l+1)}]\nonumber\\
&=\mathbbm{1}_{\mZ^{(l+1)}>0}\odot\E[\nabla \mH^{(l+1)}],
\end{align*}

where $\odot$ is the element-wise product and $\mathbbm{1}$ is the indicator function.
The element-wise product is linear operation and $\mathbbm{1}_{\mZ^{(l+1)}>0}$ is only related to the pre-activation $\mZ^{(l+1)}$ in the forward pass.
We only apply the approximation during the backward pass
so $\mathbbm{1}_{\mZ^{(l+1)}>0}$ can be extracted from the expectation.
We know that for the last layer $L$, we have $\E[\nabla \mH^{(L)}]=\mH^{(L)}$ since we do not apply activation at the output layer.
We then can prove by induction that $\E[\nabla \mH^{(l+1)}]=\mH^{(l+1)}$ and $\E[\nabla\mW^{(l)}]=\mW^{(l)}$ for any layer $l$.

\section{Proof}
\label{app: theory}

\subsection{Derivation of \Eqref{eq: col_row_norm}}

Let $\mX\in\mathbb{R}^{n\times m}$,  $\mY\in\mathbb{R}^{m\times q}$ be two matrices.
The matrix multiplication $\mX\mY$ can be estimated as 
\begin{align}
\mm(\mX, \mY) &= \sum_{i=1}^{m} \mX_{:,i} \mY_{i,:}\approx \sum_{t=1}^{k} \frac{1}{kp_{i_t}}\mX_{:,i_t}\mY_{i_t,:} =\mX'\mY', \nonumber
\end{align}

\Eqref{eq: col_row_norm} shows the approximation error $\E[||\mX\mY-\mX'\mY'||_F]$ is minimized when the probabilities
\begin{equation}
    p_i = \frac{ ||\mX_{:,i}||_2\ ||\mY_{i,:}||_2}{\sum_{j=1}^{m} ||\mX_{:,j}||_2\ ||\mY_{j,:}||_2}. \nonumber
\end{equation}

\begin{proof}
    
Let $f(i)=\frac{\mX_{:i}\mY_{i:}}{p_{i}}\in\mathbb{R}^{n\times q}$. 
We note that $f(i)$ is an unbiased estimation of $\mX\mY$. Namely,
\begin{align}
\E_{j\sim \mathcal{P}}[f(j)]=\sum_{i=1}^m  p_{i}\frac{\mX_{:,i}\mY_{i:}}{p_{i}}=\mX\mY.   \nonumber 
\end{align}

Then we have 
\begin{equation}
\mX'\mY'=\frac{1}{k}\sum_{t=1}^k f(i_t),
\label{eq: eq_app1}
\end{equation}
where $i_1,\cdots,i_t$ are the index of the sampled column-row pairs at $t^{\mathrm{th}}$ random trials.
For each $i_t$, its variance is 

\begin{align}
    \Var[f(i_t)] &= \Var[\frac{\mX_{:i_t}\mY_{i_t:}}{p_{i_t}}] \nonumber \\
    &= \E[\frac{\mX^2_{:i_t}\mY^2_{i_t:}}{p^2_{i_t}}] - \E^2[\frac{\mX_{:i_t}\mY_{i_t:}}{p_{i_t}}] \nonumber \\
    &= \E[\frac{\mX^2_{:i_t}\mY^2_{i_t:}}{p^2_{i_t}}] - (\mX\mY)^2. \nonumber \\
    &= \sum_{t=1}^{m}\frac{\mX^2_{:t}\mY^2_{t:}}{p_t} - (\mX\mY)^2. 
\label{eq: eq_app2}
\end{align}
where the first step follows from the fact that $\Var[\vx]=\E[\vx^2]-\E^2[\vx]$.

Then we have, 


\begin{align}
    \E[||\mX\mY-\mX'\mY'||_F] &= \sum_{i=1}^{n}\sum_{j=1}^{q} \E[(\mX\mY-\mX'\mY')_{ij}^2] \nonumber \\
    &= \sum_{i=1}^{n}\sum_{j=1}^{q} \Var[(\mX'\mY')_{ij}]. \nonumber
\end{align}
By combining \Eqref{eq: eq_app1} and \Eqref{eq: eq_app2} into the above equation, we have
\begin{align}
    \E[||\mX\mY-\mX'\mY'||_F] &= \frac{1}{k}\sum_{i=1}^{n}\sum_{j=1}^{q} \sum_{t=1}^{m}\frac{\mX^2_{it}\mY^2_{tj}}{p_t} - \frac{1}{k}\|\mX\mY\|_F^2. \nonumber \\
     &=\frac{1}{k} \sum_{t=1}^m \frac{\|\mX_{:,t}\|_2^2 \|\mY_{t,:}\|_2^2}{p_t} - \frac{1}{k}\|\mX\mY\|_F^2. \nonumber
\end{align}

To minimize $\E[||\mX\mY-\mX'\mY'||_F]$, the optimal probability distribution can be obtained via solving the following optimization problem:

\begin{align}
    \min_{p_1,\cdots, p_m}  & \sum_{t=1}^m \frac{\|\mX_{:,t}\|_2^2 \|\mY_{t,:}\|_2^2}{p_t}, \nonumber \\
     \text{s.t.} & \sum_{t=1}^m p_t = 1. \nonumber
\end{align}

The solution to the above convex problem is the distribution defined in \Eqref{eq: col_row_norm}. Namely,

\begin{equation}
    p_i = \frac{ ||\mX_{:,i}||_2\ ||\mY_{i,:}||_2}{\sum_{j=1}^{m} ||\mX_{:,j}||_2\ ||\mY_{j,:}||_2}. \nonumber
\end{equation}

\end{proof}

\subsection{Unbiasedness of Our Proposed Estimator}
\label{app: theory_unbias}
\theounbias*
\begin{proof}

\begin{align}
&\E_{j\sim \mathcal{P}^{\mathcal{D}\backslash\mathcal{C}}} \Big[\sum_{c\in\mathcal{C}} f(c)p_c + (1 - \sum_{c\in\mathcal{C}}p_c)f(j)\Big] \nonumber \\
=&\sum_{c\in\mathcal{C}} f(c)p_c  + (1- \sum_{c\in\mathcal{C}} p_c)\E_{j\sim \mathcal{P}^{\mathcal{D}\backslash\mathcal{C}}} [f(j)] \nonumber \\
=&\sum_{c\in\mathcal{C}} f(c)p_c  + (1- \sum_{c\in\mathcal{C}} p_c)\sum_{j\in\mathcal{D}\backslash\mathcal{C}}\frac{p_j}{1- \sum_{c\in\mathcal{C}} p_c} f(j)\nonumber\\
=&\sum_{c\in\mathcal{C}} f(c)p_c + \sum_{j\in\mathcal{D}\backslash\mathcal{C}}f(j)p_j\nonumber\\
=&\E_{j\sim\mathcal{P}}[f(j)]\nonumber \\
=&\mX\mY\nonumber
\end{align}
\end{proof}

\subsection{Variance of Our Proposed Estimator}
\label{app: theory_var}

\theounvar*

\begin{proof}

Recall that the original estimator for matrix production $\mX\mY$ is defined as 
\begin{align}
\E_{i\sim \mathcal{P}}[f(i)]. 
\label{eq: ori_estimator}
\end{align}

and our proposed family of estimator is defined as:
\begin{align}
h(j)=\E_{j\sim \mathcal{P}^{\mathcal{D}\backslash\mathcal{C}}} \Big[\sum_{c\in\mathcal{C}} f(c)p_c + (1 - \sum_{c\in\mathcal{C}}p_c)f(j)\Big].
\label{eq: our_estimator_single_form_app}
\end{align}

We first define three independent random variables as belows:
\begin{align}
    u &\sim \mathcal{P}^{\mathcal{C}},\\
    j &\sim \mathcal{P}^{\mathcal{D}\backslash\mathcal{C}},\\
    b &\sim \text{Bernoulli}(1-\sum_{c\in\mathcal{C}}p_c).
\end{align}

According to the Law of total variance, we have
\begin{align}
        \Var[f(i)] &= \E_b\Big[Var[f(i)|b]\Big] + \Var_b\Big[\E[f(i)|b]\Big] \nonumber \\
    &\geq \E_b\Big[\Var[f(i)|b]\Big] \nonumber \\
    &= \sum_{c\in\mathcal{C}}p_c \Var[f(i)|b=0] + (1-\sum_{c\in\mathcal{C}}p_c) \Var[f(i)|b=1] \nonumber \\
    &\geq (1-\sum_{c\in\mathcal{C}}p_c) \Var[f(i)|i\in \mathcal{D}\backslash\mathcal{C}]
\end{align}
where the first step follows from the fact that for any random variance $\vx,\vy$, we have $\Var[\vy]=\E[\Var[\vy|\vx]]+\Var[\E[\vy|\vx]]$.
Also, by \Eqref{eq: our_estimator_single_form_app}, we have

\begin{equation}
    \Var[h(j)]=(1-\sum_{c\in\mathcal{C}}p_c)^2 \Var[f(j)|j\in \mathcal{D}\backslash\mathcal{C}].
\end{equation}

By combining the above two inequality, we have
\begin{equation}
    \Var[h(j)] \leq  (1-\sum_{c\in\mathcal{C}}p_c) \Var[f(i)].
    \label{eq: tmp_res_1}
\end{equation}

\Eqref{eq: tmp_res_1} quantitatively shows the variance reduction of $h(j)$ over $f(i)$.
Then we compare our estimator $\hat{g}(\mX, \mY)$ and $g(\mX, \mY)$ in terms of variance.

First, because 
$g(\mX,\mY) = \frac{1}{k}\sum_{t=1}^k f(i_t), ~~~i_1,\cdots i_k \overset{\text{i.i.d}}{\sim} \mathcal{P}$. Thus we have 
\begin{equation}
    \Var[g(\mX,\mY)] = \frac{1}{k} \Var[f(i)].
\end{equation}

Similarly, we have
\begin{equation}
    \Var[\hat{g}(\mX,\mY)] = \frac{1}{k-|\mathcal{C}|} \Var[h(j)].
\end{equation}

By combining \Eqref{eq: tmp_res_1} into the above two equations, we have
\begin{align}
     \Var[\hat{g}(\mX,\mY)] &= \frac{1}{k-|\mathcal{C}|} \Var[h(j)] \\ \nonumber
     &\leq \frac{1-\sum_{c\in\mathcal{C}}p_c}{k-|\mathcal{C}|} \Var[f(i)] \\ \nonumber
     &\leq \frac{1-\sum_{c\in\mathcal{C}}p_c}{k-|\mathcal{C}|}k \Var[g(\mX,\mY)], \nonumber
\end{align}

where the first step follows from \Eqref{eq: tmp_res_1}.
By setting $\frac{1-\sum_{c\in\mathcal{C}}p_c}{k-|\mathcal{C}|}k\leq 1$, we arrive the conclusion that when $\sum_{c\in\mathcal{C}}p_c > \frac{|\mathcal{C}|}{k}$, we have $ \Var[\hat{g}(\mX,\mY)]\leq \Var[g(\mX,\mY)]$.

Further, $\frac{1-\sum_{c\in\mathcal{C}}p_c}{k-|\mathcal{C}|}k$ achieves the minimal when $|\mathcal{C}|=\min_{|\mathcal{C}|\in\{0,\cdots,k\}}\frac{1 - \sum_{c\in\mathcal{C}}p_c}{k-|\mathcal{C}|}$.







\end{proof}

\section{Implementation Details}
\label{app: implementation}
The pseudocode for approximated linear layer with \sas and standard line layer is given in Algorithm \ref{algo: approx_q_linear} and Algorithm \ref{algo: q_linear}, respectively.
The column-row pair sampling procedure is given in Algorithm \ref{algo: subsample}.
For the ease of illustration, we ignore the sequential length.
As we mentioned in the main text, we only replace the \mm in the backward pass with \sas.
According to \Eqref{eq: bwd_w}, we need the activation gradient $\nabla\mZ$ to perform the column-row pair sampling during the forward pass.
Thus we initialize a cache in CPU memory to store the gradient norm of activations from the last step.
When performing column-row pair selection, we need to swap the gradient norm of activations between CPU and GPU, which will cause extra time overhead due to the data movement.
Fortunately, we note that the number of elements in the gradient norm of activations is significantly less than the one in activations, which does not cause a significant time overhead.

\begin{algorithm}[h!]  
 \SetKwInput{KwParam}{Hyperparameter}
 \KwParam{The total budget of column-row pairs $k$.}
  \SetKwProg{myProcedure}{procedure}{\string:}{end procedure}
  \myProcedure{\textsc{Init}}{
  Initialize $\texttt{Cache}\in\mathbb{R}^{N}$ as an empty matrix in main memory \tcp{$N$ is the total number of samples in the dataset. \texttt{Cache} is used for saving the norm of output gradient $\nabla\mZ$.}
  }
  \myProcedure{\textsc{Forward Pass}}{
  \KwIn{activation $\mH\in\mathbb{R}^{B\times D}$, weight $\mW\in\mathbb{R}^{D\times D}$, indices of the current batch samples $BI=\{j_{1},\cdots,j_B\}$.}
  $\texttt{ctx}\leftarrow{\{\}}$\hspace{.5em}\tcp{the context which saves tensors for backward}
$\mZ = \mH\mW$\\
$\mH', ind \leftarrow$\textsc{Subsample}($\mH$, \texttt{Cache}[$BI$], $k$) \\ \tcp{\texttt{Cache}[$BI$] is the cached gradient norm from the backward pass; $ind$ is the set of involved column-row pair indices}x
$\texttt{ctx}\leftarrow\{\mH', \mW, BI, ind\}$\\
\Return{$\mZ$}
  }
  \myProcedure{\textsc{Backward pass}}{
  \KwIn{$\texttt{ctx}$ from the forward pass, output gradient $\nabla\mZ\in\mathbb{R}^{B\times D}$}
$\mH', \mW, BI, ind \leftarrow\texttt{ctx}$\\
$\nabla\mH=\nabla\mZ\mW^\top$\\
$\nabla\mZ'\leftarrow \nabla\mZ[ind]$\\ \tcp{$\nabla\mZ'\in\mathbb{R}^{k \times D}$}
$\nabla\mW = \mH'^\top\nabla\mZ'$\\
\For{j \rm{in} $BI$}{
\texttt{Cache}[$j$] = $\|\nabla\mZ_{j,:}\|_2$
}
\tcp{Update the gradient norm of samples in the current batch}
\Return{$\nabla\mH, \nabla\mW$}
}
  \caption{Forward \& Backward pass of Approximated Linear Layer}
  \label{algo: approx_q_linear}
\end{algorithm}

\begin{algorithm}
  \KwIn{activation $\mH\in\mathbb{R}^{B\times D}$, gradient norm $\vz\in\mathbb{R}^B$, the total budget of column-row pairs $k$.}\
  \For{$i=1,\cdots,B$}{
  $p_i \leftarrow \frac{ \vz_i ||\mH_{i,:}||_2}{\sum_{j=1}^{B} \vz_i ||\mH_{j,:}||_2}$ \tcp{The probability of column-row pairs defined in \Eqref{eq: col_row_norm}.} 
  }
  $\hat{k}\leftarrow \min_{\hat{k}\in\{0,\cdots,k\}}\frac{1 - \sum_{c\in\mathcal{C}}p_c}{k-\hat{k}}$, s.t. $\mathcal{C}=|\hat{k}|$. \tcp{$\mathcal{C}$ is the set of column-row pair indices associated with $|\mathcal{C}|$ largest $p_i$.} 
  Sample $k-|\mathcal{C}|$ i.i.d. column-row pairs $\mathcal{C}_{\text{stoc}}=\{i_1,\cdots,i_{k-|\mathcal{C}|}\}$ from the distribution $\mathcal{P}^{\mathcal{D}\backslash\mathcal{C}}$\\
  $ind\leftarrow \mathcal{C} \cup \mathcal{C}_{\text{stoc}}$\\
 \For{$j\in \mathcal{C}_{\text{stoc}}$}{
$\mH[j,:]\leftarrow \mH[j,:] * \frac{1 - \sum_{c\in\mathcal{C}}p_c}{(k-|\mathcal{C}|)p_j}$~~~~\tcp{We need to normalize the stochastic part in \Eqref{eq: our_estimator} to ensure the unbiasedness.}
 }
  $\mH'\leftarrow \mH[ind]$~~~~\tcp{$\mH'\in\mathbb{R}^{k \times D}$}
  \Return{$\mH'$, ind}
  \caption{\textsc{Subsample}}
  \label{algo: subsample}
\end{algorithm}

\begin{algorithm}
  \SetKwProg{myProcedure}{procedure}{\string:}{end procedure}
  \myProcedure{\textsc{Forward Pass}}{
  \KwIn{activation $\mH_Q\in\mathbb{R}^{BS\times D}$, weight $\mW_Q\in\mathbb{R}^{D\times D}$, batch indices $index$}
  $\texttt{ctx}\leftarrow{\{\}}$\hspace{.5em}\tcp{the context which saves tensors for backward}
$\mZ_Q = \mH_Q\mW_Q$\\
$\texttt{ctx}\leftarrow\{\mH_Q, \mW_Q\}$\\
\Return{$\mZ_Q$}
  }
  \myProcedure{\textsc{Backward pass}}{
  \KwIn{$\texttt{ctx}$ from the forward pass, output gradient $\nabla\mZ_Q$}
$\mH_Q, \mW_Q \leftarrow\texttt{ctx}$\\
$\nabla\mH_Q=\nabla\mZ_Q\mW_Q^\top$\\
$\nabla\mW_Q = \mH_Q^\top\nabla\mZ_Q$\\
\Return{$\nabla\mH_Q, \nabla\mW_Q$}
}
  
  \caption{Forward \& Backward pass of the standard Linear layer}
  \label{algo: q_linear}
\end{algorithm}

\clearpage
\section{More Experimental Results}
\label{app: more_res}

\begin{figure*}[hbt!]
    \centering
    \begin{subfigure}[h]{0.34\linewidth}
    \includegraphics[width=\linewidth]{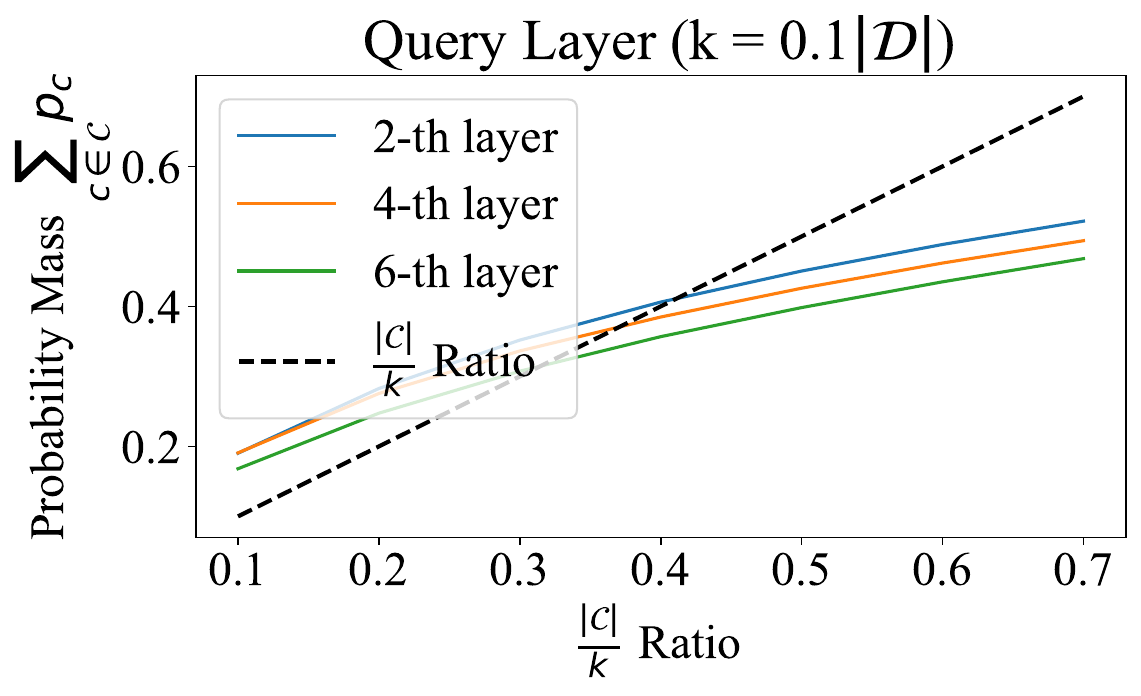}
    \end{subfigure}
    \begin{subfigure}[h]{0.31\linewidth}
    \includegraphics[width=\linewidth]{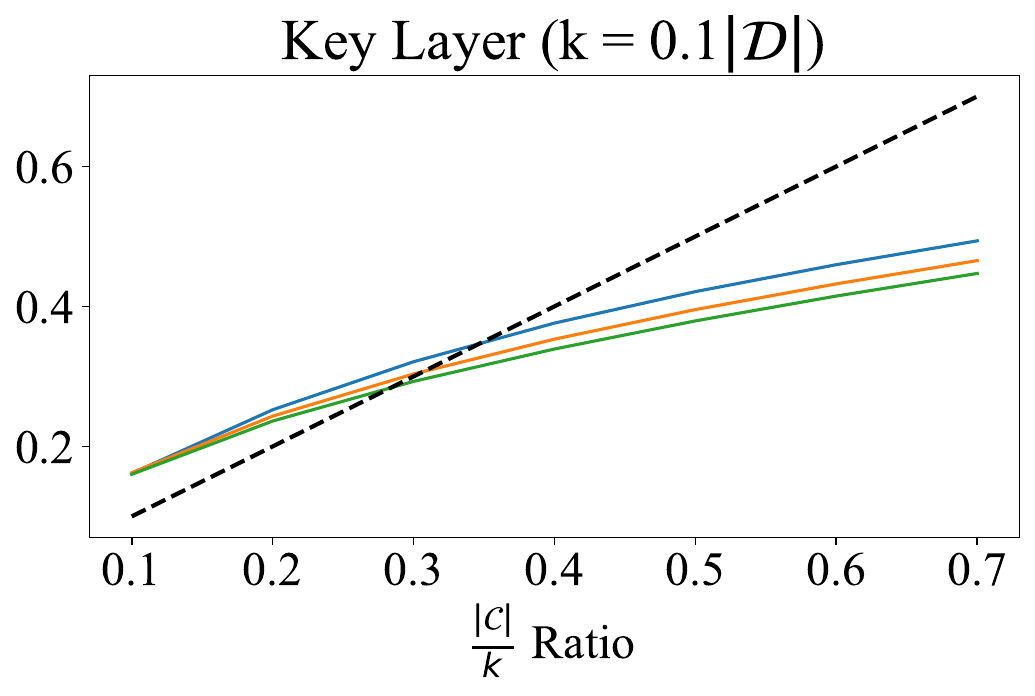}
    \end{subfigure}
    \begin{subfigure}[h]{0.31\linewidth}
    \includegraphics[width=\linewidth]{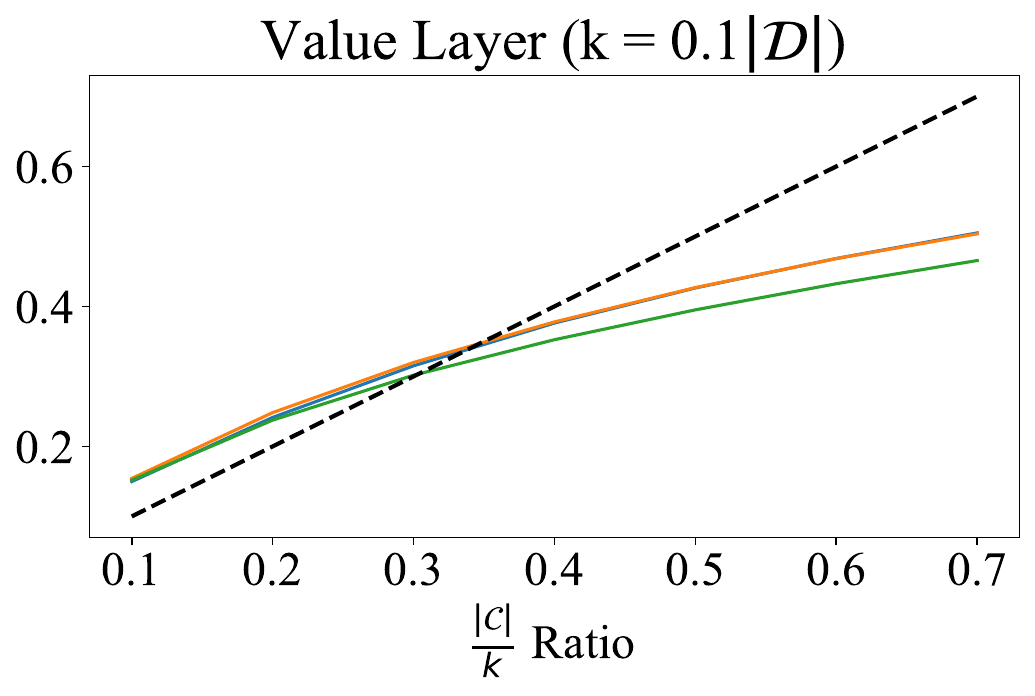}
    \end{subfigure}
    \caption{The probability mass $\sum_{c\in\mathcal{C}}p_c$ versus $\frac{|\mathcal{C}|}{k}$ in \Eqref{eq: var_thresh} at $k=0.1|\mathcal{D}|$.
    Here we visualize the column-row index distribution of query/key/value layer T5-base model, fine-tuned on RTE dataset.}
    \label{fig: exp_assump_analysis_th2_direct_k_0.1}
\end{figure*}

\begin{figure*}[hbt!]
    \centering
    \begin{subfigure}[h]{0.34\linewidth}
    \includegraphics[width=\linewidth]{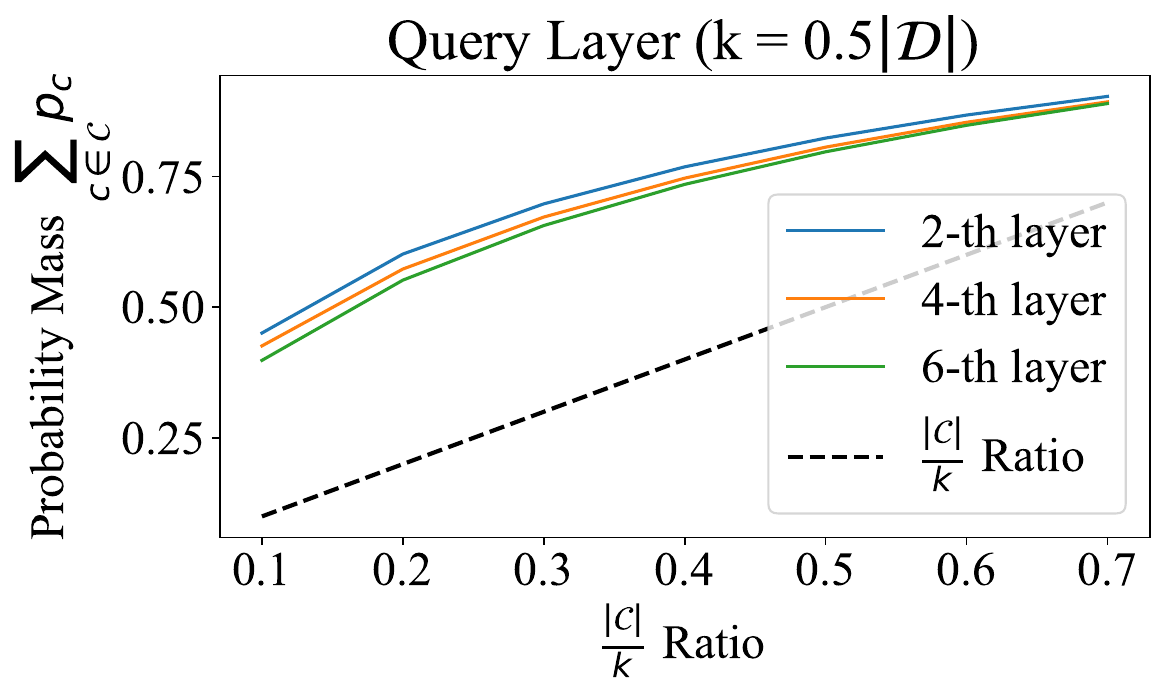}
    \end{subfigure}
    \begin{subfigure}[h]{0.31\linewidth}
    \includegraphics[width=\linewidth]{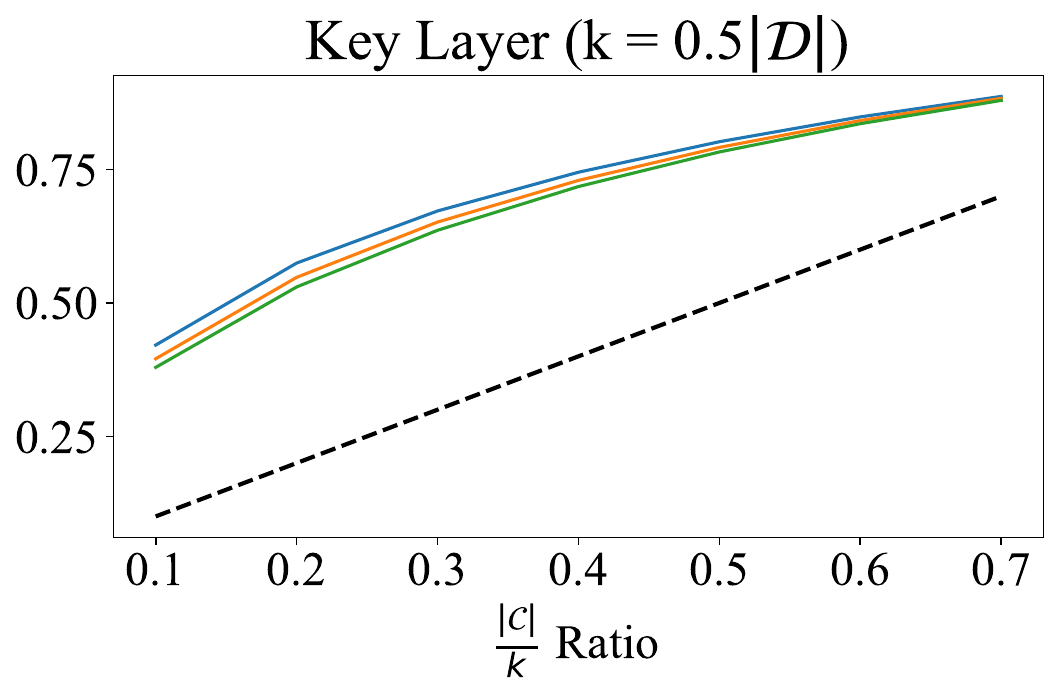}
    \end{subfigure}
    \begin{subfigure}[h]{0.31\linewidth}
    \includegraphics[width=\linewidth]{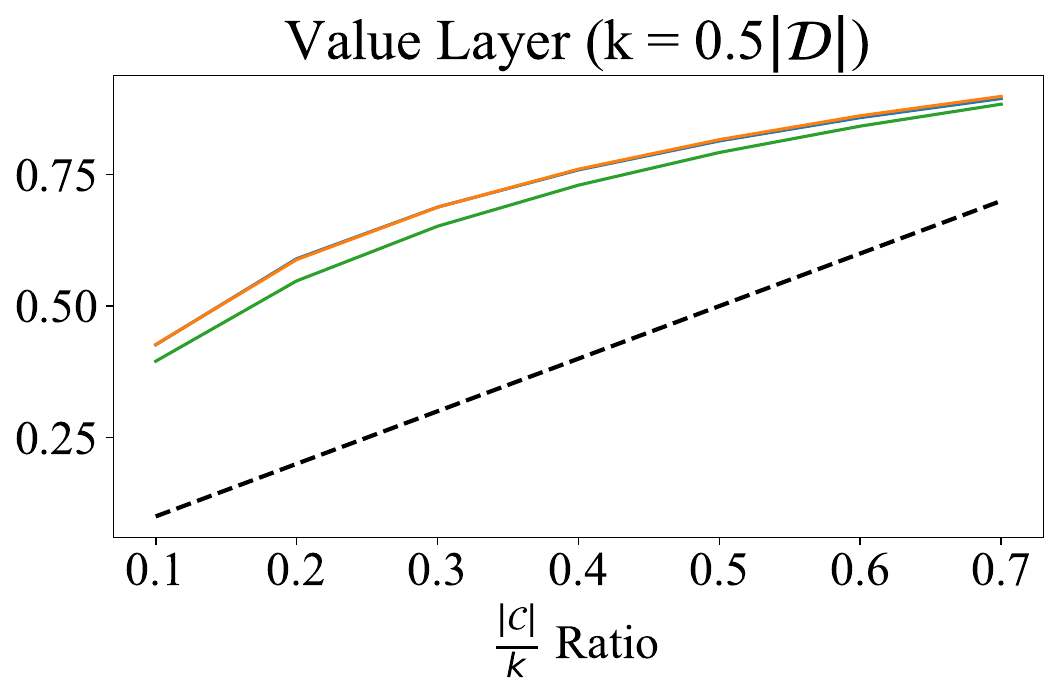}
    \end{subfigure}
    \caption{The probability mass $\sum_{c\in\mathcal{C}}p_c$ versus $\frac{|\mathcal{C}|}{k}$ in \Eqref{eq: var_thresh} at $k=0.5|\mathcal{D}|$.
    Here we visualize the column-row index distribution of query/key/value layer T5-base model, fine-tuned on RTE dataset.}
    \label{fig: exp_assump_analysis_th2_direct_k_0.5}
\end{figure*}

\begin{figure*}[hbt!]
    \centering
    \begin{subfigure}[h]{0.34\linewidth}
    \includegraphics[width=\linewidth]{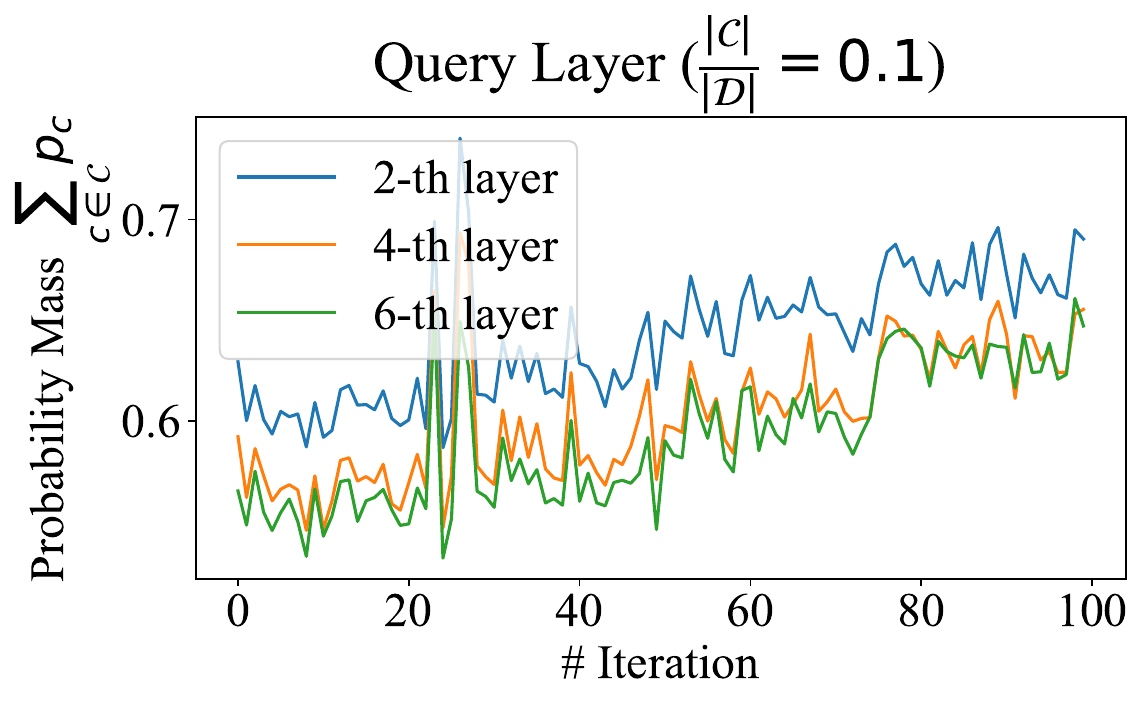}
    \end{subfigure}
    \begin{subfigure}[h]{0.31\linewidth}
    \includegraphics[width=\linewidth]{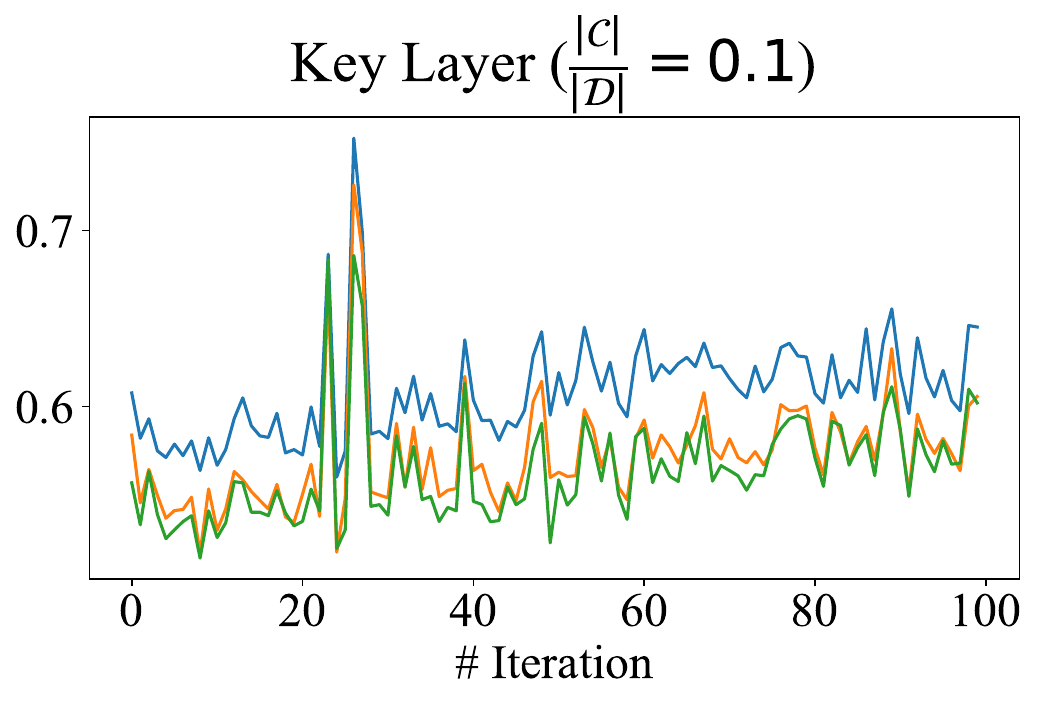}
    \end{subfigure}
    \begin{subfigure}[h]{0.31\linewidth}
    \includegraphics[width=\linewidth]{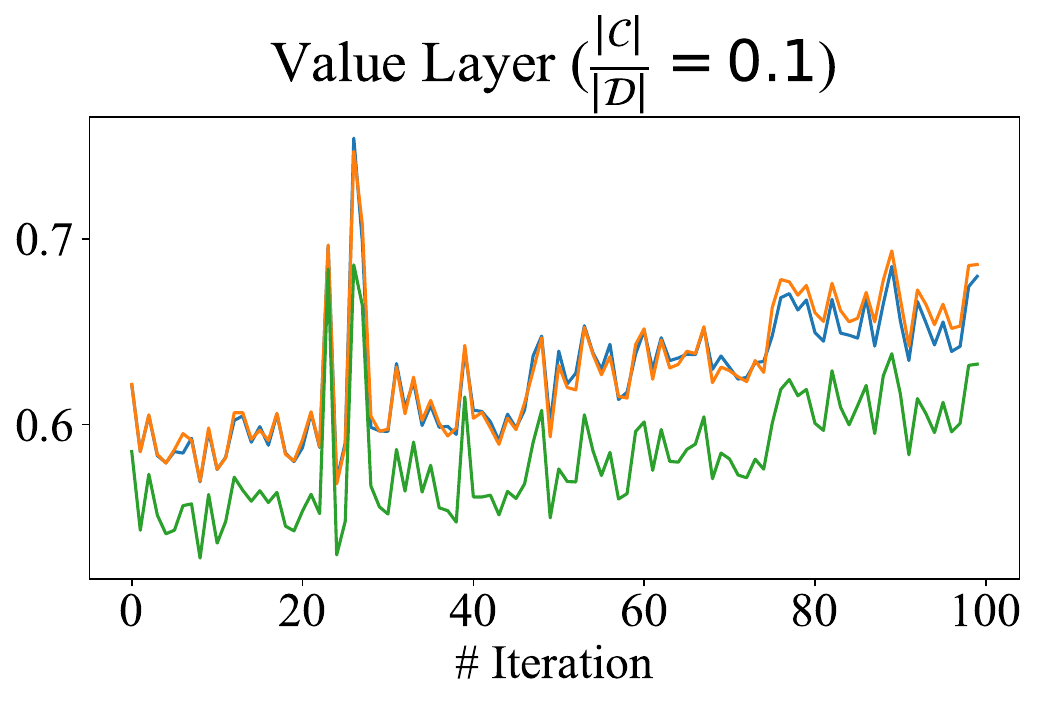}
    \end{subfigure}
    \caption{The probability mass of top-$10\%$ column-row pairs in \Eqref{eq: col_row_norm} versus iterations.
    Here we visualize the query/key/value layer T5-base model, fine-tuned on RTE dataset.}
    \label{fig: exp_assump_analysis_versus_iterations}
\end{figure*}
\subsection{More Experimental Analysis on Theorem \ref{theo: var}}
\label{app: more_res_theo2}
To evaluate Theorem \ref{theo: var} more comprehensively, 
below we also plot the $\sum_{c\in\mathcal{C}}p_c$ versus $\frac{|\mathcal{C}|}{k}$ at $k=0.1|\mathcal{D}|$ and $k=0.5|\mathcal{D}|$ in Figure \ref{fig: exp_assump_analysis_th2_direct_k_0.1} and \ref{fig: exp_assump_analysis_th2_direct_k_0.5}, respectively.
We also plot $\sum_{c\in\mathcal{C}}p_c$ versus iterations in Figure \ref{fig: exp_assump_analysis_versus_iterations}.
We summarize that the the column-row index distribution is concentrated on a few column-row pairs. Thus, the assumption in Theorem \ref{theo: var} holds under the context of fine-tuning transformers.

\subsection{More Experimental Speed Analysis}
\label{app:more_exp_speed}
Increasing the batch size can often result in faster model convergence and/or enhance the final performance. 
Ideally, we should adjust the batch size according to the requirements of our model rather than being constrained by the GPU's memory capacity. 
To illustrate this, we have represented the correlation between peak memory usage and maximum mini-batch size for T5-Base, T5-Large, and T5-3B in Figure \ref{fig:more_mem_vs_bs}. Our observations highlight that \sas effectively increases the maximum available batch size.

\nocite{zhong2022revisit}
\nocite{DBLP:journals/corr/abs-2010-13015}
\nocite{liu2021divaug}

We also provide the apple-to-apple speed comparison for linear operation with and without \sas in Table \ref{tab: latency}.
In Table \ref{tab: latency}, ``Fwd'', ``Bwd'', and ``F-B'' are the time of forward pass,  the time of backward pass, and the total time for both the forward and backward pass, respectively. 
We summarize that under the same workload, the current implementation of \sas may roughly slow down the linear operation about $20\%$.
This is because the extra sampling process and data movement counteract the acceleration (see Algorithm \ref{algo: approx_q_linear}).
However, we note that (1) the overhead can be greatly reduced with better implementation, e.g., using prefetch and operation-fusion technique \citep{gact}; (2)  the existing implementation can still yield a large speedup when employing larger batch sizes (Figure \ref{fig:throughput}).

\begin{figure}[h!]
    \centering
    \begin{subfigure}[h]{0.3\linewidth}
      \includegraphics[width=1\linewidth]{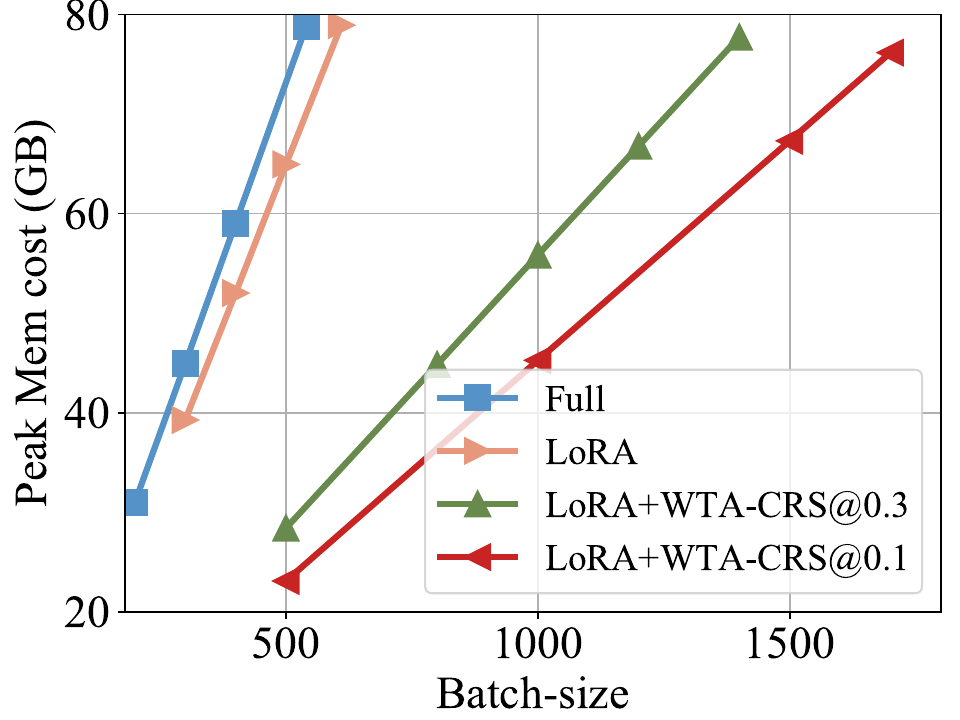}
      \caption{T5-Base}
    \end{subfigure}%
    \begin{subfigure}[h]{0.3\linewidth}
      \includegraphics[width=1\linewidth]{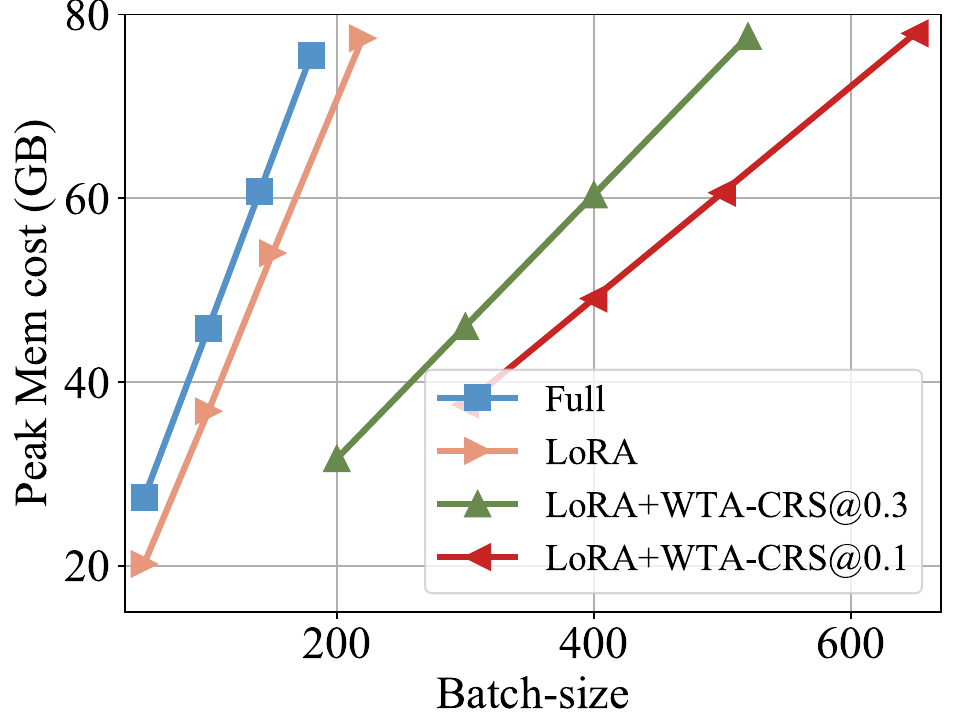}
      \caption{T5-Large}
    \end{subfigure}%
    \begin{subfigure}[h]{0.3\linewidth}
      \includegraphics[width=1\linewidth]{figures/t5-3b_mem_vs_batch.pdf}
      \caption{T5-3B}
    \end{subfigure}%
    \caption{Peak memory usage versus maximum mini-batch size of T5.}
    \label{fig:more_mem_vs_bs}
\end{figure}

\begin{table}[]
\centering
    \begin{tabular}{llcccccc}
    \toprule
          & Method & \makecell[c]{T5- \\ ATT} & \makecell[c]{T5- \\ FF} & \makecell[c]{T5 \\ Block} & \makecell[c]{T5- \\ Large} \\
    \midrule
         \multirow{2}{*}{Fwd} & Full & 8 & 10 & 17 & 1052 \\
         & \Algnameabbr{} & 22 & 16 & 37 & 2013 \\
    \midrule
         \multirow{2}{*}{Bwd} & Full & 16 & 19 & 34 & 2073 \\
         & \Algnameabbr{} & 15 & 14 & 30 & 1738 \\
    \midrule
         \multirow{2}{*}{F-B} & Full & 24 & 29 & 51 & 3125 \\
         & \Algnameabbr{} & 37 & 30 & 67 & 3751 \\
    \bottomrule
    \end{tabular}
    \vspace{1em}
  \captionof{table}{\small Latency~(ms) of Forward and Backward pass.}
  \label{tab: latency}
\end{table}

\section{Experimental Settings}
\label{appendix:hyper_setting}

We give the detailed hyper-parameter setting in this section.
Specifically, for both T5 and BERT models, the parameters are updated with the AdamW optimizer $\beta_1 = 0.9$ $\beta_2 = 0.999$ $\epsilon = 10^{-8}$ and $\text{weight decay} = 0 $.
The the learning rate is adjusted with a linear LR Scheduler, which maintains a constant learning rate for the initial 500 steps, and adjusts it gradually thereafter.
The input sequences are padded to the maximum length 128.
\Algnameabbr{} has a LoRA dimension 32 if it is combined with LoRA.
To achieve the optimal solution, the T5-Base, Large, 3B and BERT-Base and Large models have different learning rate, training epoch number, and mini-batch size on different datasets, which are given in Tables~\ref{tab:lr_parameter}, \ref{tab:epoch_parameter}, \ref{tab:bs_parameter}, respectively.

\subsection{Computational Infrastructure}
\label{appendix:hardware}

The computational infrastructure information is given in Table \ref{tab:computing_infrastructure}.

\begin{table}[h]
\centering
\caption{Computing infrastructure for the experiments.}
\begin{tabular}{l|c}
\toprule
Device Attribute & Value \\
\hline
Computing infrastructure & GPU \\
GPU model & NVIDIA-A100 \\ 
GPU Memory & 81251MB \\ 
CUDA Version & 11.4 \\
CPU Memory & 512GB \\
\bottomrule
\end{tabular}
\label{tab:computing_infrastructure}
\end{table}

\begin{table}[]
\caption{Learning rate.}
\label{tab:lr_parameter}
\resizebox{\textwidth}{!}{
\begin{tabular}{llcccccccc}
\toprule
Model                       & Method           & CoLA          & SST-2         & MRPC          & QQP           & MNLI          & QNLI          & RTE           & STS-B       \\ \hline
\multirow{2}{*}{BERT-Base}  & \Algnameabbr{}@0.3 & \multicolumn{8}{c}{2e-5} \\  
                            & LoRA+\Algnameabbr{}@0.3 & 2e-4 & 5e-4 & 2e-4 & 3e-4 & 3e-4 & 2e-4 & 2e-4 & 3e-4 \\  \midrule
\multirow{4}{*}{T5-Base}    & \Algnameabbr{}@0.3 & \multicolumn{5}{c}{3e-5} & 3e-6 & 3e-5 & 3e-5 \\ 
                            & \Algnameabbr{}@0.1 & \multicolumn{8}{c}{3e-5} \\ 
                            & LoRA+\Algnameabbr{}@0.3 & 3e-4 & 3e-5 & 3e-4 & 3e-5 & 3e-5 & 3e-5 & 3e-4 & 3e-4 \\ 
                            & LoRA+\Algnameabbr{}@0.1 & 3e-4 & 3e-5 & 3e-4 & 3e-5 & 3e-5 & 3e-5 & 3e-4 & 3e-4 \\ \midrule
\multirow{2}{*}{BERT-Large} & \Algnameabbr{}@0.3 & \multicolumn{8}{c}{2e-5} \\ 
                            & LoRA+\Algnameabbr{}@0.3 & 3e-4 & 2e-4 & 2e-4 & 2e-4 & 2e-4 & 2e-4 & 3e-4 & 3e-4 \\ \midrule
\multirow{2}{*}{T5-Large}   & \Algnameabbr{}@0.3 & \multicolumn{5}{c}{3e-5} & 3e-6 & 3e-5 & 3e-5 \\ 
                            & \Algnameabbr{}@0.1 & \multicolumn{5}{c}{3e-5} & 3e-6 & 3e-5 & 3e-5 \\ 
                            & LoRA+\Algnameabbr{}@0.3 & 3e-4 & 3e-5 & 3e-4 & 3e-5 & 3e-5 & 3e-5 & 3e-4 & 3e-4 \\ 
                            & LoRA+\Algnameabbr{}@0.1 & 3e-4 & 3e-5 & 3e-4 & 3e-5 & 3e-5 & 3e-5 & 3e-4 & 3e-4 \\ \midrule
\multirow{2}{*}{T5-3B}      & LoRA+\Algnameabbr{}@0.3 & 3e-4 & 3e-5 & 3e-4 & 3e-4 & 3e-4 & 3e-5 & 3e-4 & 3e-4 \\ 
                            & LoRA+\Algnameabbr{}@0.1 & 3e-4 & 3e-5 & 3e-4 & 3e-4 & 3e-4 & 3e-5 & 3e-4 & 3e-4 \\ \bottomrule
\end{tabular}
}
\end{table}

\begin{table}[]
\caption{Training epoch number.}
\resizebox{\textwidth}{!}{
\label{tab:epoch_parameter}
\begin{tabular}{llcccccccc}
\toprule
Model                       & Method           & CoLA          & SST-2         & MRPC          & QQP           & MNLI          & QNLI          & RTE           & STS-B       \\ \midrule
\multirow{2}{*}{BERT-Base}  & \Algnameabbr{}@0.3 & 20 & 20 & 10 & 10 & 10 & 10 & 20 & 10 \\ 
                            & LoRA+\Algnameabbr{}@0.3 & 60 & 20 & 20 & 20 & 20 & 20 & 40 & 40 \\  \hline
\multirow{4}{*}{T5-Base}    & \Algnameabbr{}@0.3 & 40 & 10 & 20 & 10 & 10 & 10 & 50 & 20 \\ 
                            & \Algnameabbr{}@0.1 & 40 & 10 & 20 & 10 & 10 & 10 & 50 & 20 \\ 
                            & LoRA+\Algnameabbr{}@0.3 & 40 & 10 & 20 & 20 & 20 & 10 & 50 & 20 \\ 
                            & LoRA+\Algnameabbr{}@0.1 & 40 & 10 & 20 & 20 & 20 & 10 & 50 & 20 \\ \midrule
\multirow{2}{*}{BERT-Large} & \Algnameabbr{}@0.3 & 60 & 20 & 20 & 10 & 10 & 10 & 40 & 10 \\                               & LoRA+\Algnameabbr{}@0.3 & 60 & 20 & 20 & 20 & 20 & 20 & 40 & 40 \\ \midrule
\multirow{4}{*}{T5-Large}   & \Algnameabbr{}@0.3 & 20 & 10 & 20 & 10 & 10 & 10 & 40 & 20 \\ 
                            & \Algnameabbr{}@0.1 & 20 & 10 & 20 & 10 & 10 & 10 & 40 & 20 \\
                            & LoRA+\Algnameabbr{}@0.3 & 40 & 10 & 40 & 10 & 10 & 10 & 60 & 20 \\
                            & LoRA+\Algnameabbr{}@0.1 & 40 & 10 & 20 & 10 & 10 & 10 & 60 & 20 \\ \midrule
\multirow{2}{*}{T5-3B}      & LoRA+\Algnameabbr{}@0.3 & 40 & 10 & 20 & 10 & 10 & 10 & 60 & 20 \\
                            & LoRA+\Algnameabbr{}@0.1 & 40 & 10 & 20 & 10 & 10 & 10 & 60 & 20 \\ \bottomrule
\end{tabular}
}
\end{table}

\begin{table}[]
\caption{Training mini-batch size.}
\label{tab:bs_parameter}
\resizebox{\textwidth}{!}{
\begin{tabular}{llcccccccc}
\hline
Model                       & Method           & CoLA          & SST-2         & MRPC          & QQP           & MNLI          & QNLI          & RTE           & STS-B       \\ \toprule
\multirow{2}{*}{BERT-Base/Large}   & \Algnameabbr{}@0.3 & \multicolumn{7}{c}{128} & 16 \\ 
                            & LoRA+\Algnameabbr{}@0.3 & \multicolumn{7}{c}{128} & 16 \\ \midrule
\multirow{4}{*}{T5-Base/Large/3B}    & \Algnameabbr{}@0.3 & \multicolumn{8}{c}{100} \\ 
                            & \Algnameabbr{}@0.1 & \multicolumn{8}{c}{100} \\ 
                            & LoRA+\Algnameabbr{}@0.3 & \multicolumn{8}{c}{100} \\ 
                            & LoRA+\Algnameabbr{}@0.1 & \multicolumn{8}{c}{100} \\ \bottomrule
\end{tabular}
}
\end{table}

\end{document}